\pdfoutput=1
\documentclass{article}

\PassOptionsToPackage{square,numbers}{natbib}

\usepackage[preprint]{rica}

\usepackage[utf8]{inputenc}
\usepackage[T1]{fontenc}
\usepackage{hyperref}
\usepackage{url}
\usepackage{booktabs}
\usepackage{multirow}
\usepackage{amsmath}
\usepackage{amssymb}
\usepackage{amsfonts}
\usepackage{amsthm}
\usepackage{thmtools,thm-restate}
\usepackage{nicefrac}
\usepackage{microtype}
\usepackage{xcolor}
\usepackage{graphicx}
\usepackage{wrapfig}
\usepackage{placeins}
\usepackage{cleveref}
\usepackage{subcaption}
\usepackage{algorithm}
\usepackage{algorithmic}
\usepackage{bm}
\usepackage{listings}
\usepackage{tikz}
\usetikzlibrary{positioning, arrows.meta, calc}

\definecolor{codebg}{RGB}{245,247,250}
\definecolor{codecomment}{RGB}{106,153,85}
\definecolor{codekw}{RGB}{86,156,214}
\definecolor{codestring}{RGB}{206,145,120}
\definecolor{codeident}{RGB}{156,220,254}
\newtheorem{theorem}{Theorem}

\newtheorem{definition}[theorem]{Definition}

\DeclareMathOperator{\Ric}{Ric}

\DeclareMathOperator{\tr}{tr}

\newcommand{\R}{\mathbb{R}}
\newcommand{\M}{\mathcal{M}}

\newcommand{\x}{\mathbf{x}}

\newcommand{\s}{\mathbf{s}}
\newcommand{\e}{\mathbf{e}}
\newcommand{\z}{\mathbf{z}}
\newcommand{\p}{\mathsf{p}}

\newcommand{\covhess}{\nabla^2}
\newcommand{\zero}{\mathbf{0}}

\renewcommand{\acksection}{\section*{Acknowledgments}}

\graphicspath{{figures/}}

\title{Disentanglement Beyond Generative Models with Riemannian ICA}

\author{%
  Edmond Cunningham \\
  University of Massachusetts Amherst \\
  \texttt{edmondcunnin@umass.edu}
}

\begin{document}

\maketitle

\begin{abstract}
  There is a gap between the theoretical foundations of disentanglement and the practice of modern representation learning.
  Existing theoretical frameworks, particularly Independent Component Analysis (ICA) and its nonlinear variants, assume a generative model with statistically independent latent variables underlying the data so that disentanglement amounts to identifying the latents that could have generated the data.
  This generative framework is interpretable and theoretically justified, but its strong modeling assumptions make it difficult to apply to modern representation learning.
  Modern pretrained encoders often learn features that exhibit disentangled properties without making generative assumptions, yet there is no general theory for interpreting these features as independent factors of variation.
  We take a step toward such a theory by introducing Riemannian ICA (RICA), which replaces ICA's global generative model with local geometric structure.
  RICA is founded on the observation that in ICA, the factors of variation underlying a data point can be understood through radial curves emanating from the point that map to axis-aligned lines in the latent space.
  We formalize this perspective using tools from Riemannian geometry and introduce our theory in a way that is consistent with the existing generative approach.
  Our main contribution is the \emph{disentanglement tensor}, which encodes a second-order notion of disentanglement that we call pointwise disentanglement.
  This tensor depends on the Hessian of the data log likelihood as well as the Ricci curvature induced by the model.
  In a controlled source recovery setting with known ground-truth sources, RICA recovers sources across several manifolds, while the success of ICA baselines depends on the coordinates used to represent the observations.
  Our work provides a theoretical basis for studying local disentanglement without assuming a global generative model.
\end{abstract}

\section{Introduction}
\label{sec:introduction}

A long-standing goal of representation learning has been to disentangle the factors of variation underlying high-dimensional data \cite{bengio2013representation}.
A representation is said to be \emph{disentangled} when a change in one factor underlying the data corresponds to a change in a single component of the representation \cite{bengio2013representation,hyvarinen2023nonlinear}.
For example, varying the pose of a face image should move exactly one component of its representation, and the components encoding eye color and hair length should remain fixed \cite{karras2019style}.
Such representations are interpretable because each component corresponds to a distinct property of the data, and it has been argued that they should enable downstream supervised learning, transfer learning, few-shot learning, and scientific analysis \cite{bengio2013representation,scholkopf2021toward,hyvarinen2023nonlinear,wang2024disentangled}.
The intuition that disentangled representations should cause semantically meaningful changes to data while simultaneously being useful for downstream tasks has motivated the study of disentanglement through the lens of generative models \cite{hyvarinen2000ica,higgins_beta-vae_2017,chen2016infogan,hyvarinen2023nonlinear,reizinger2022embrace,gresele2021independent}.
These approaches assume that data is generated from a set of independent latent variables and that the latents can also be recovered from data.
After fitting the generative model to real data, the latents can be interpreted as factors of variation directly, which enables applications such as controllable editing \cite{chen2016infogan,higgins_beta-vae_2017,wu2021stylespace,yang_causalvae_2021} and counterfactual analysis \cite{dombrowski2021diffeomorphic}, as well as theoretically justified statistical analyses of the recovered latent distribution \cite{reizinger2022embrace,chen_isolating_2018}.
Despite their interpretability, these generative approaches have not been shown to consistently deliver representations that are useful for downstream tasks in practice \cite{locatello2020sober}.
On the other hand, models based on self-supervised learning \cite{radford2021learning,oquab2024dinov,assran2023self,wei2024towards} learn features with properties that are consistent with disentanglement even though they were not explicitly built for it \cite{cunningham2023sparseautoencodershighlyinterpretable,mueller2025isolation}.
But despite their downstream utility, theoretically justified interpretations of these features remain an open problem \cite{mikolov-etal-2013-linguistic,park2024linearrepresentationhypothesisgeometry,sharkey2025open,cunningham2023sparseautoencodershighlyinterpretable,mueller2025isolation}.
In this paper, we offer one such interpretation by building upon a classical framework for generative disentanglement, called independent component analysis.

Independent component analysis (ICA) \cite{comon1994independent,hyvarinen2000ica}, and its nonlinear extensions \cite{hyvarinen2023nonlinear,buchholz2022function,gresele2021independent,zheng2022identifiability,hyvarinen2019nonlinear}, are generative frameworks for disentanglement with strong theoretical guarantees that follow from strict modeling assumptions \cite{comon1994independent,zheng2022identifiability,buchholz2022function}.
The most prevalent assumption is that the map between data and latents is globally invertible.
This assumption enables the strongest form of latent identification because each data point corresponds to a unique latent vector, and the latent distribution can be computed directly through the change of variables formula \cite{hyvarinen2000ica}.
However, the extra modeling assumptions required for this identification either use auxiliary information that is not available in every setting \cite{hyvarinen2019nonlinear} or impose rigid modeling constraints that may be misspecified in practice \cite{zheng2022identifiability,buchholz2022function,horan2021unsupervised}.
These restrictions make ICA difficult to use as a theory for modern disentangled representations, where the goal is often to interpret useful features rather than recover a globally valid generative model.
Our main insight is that the interpretation offered by generative models does not require global invertibility, but can instead be built from \emph{local invertibility} around each data point.

A generative interpretation of disentanglement requires that variations in the data be attributable to variations of individual latent components.
To understand this, observe that a variation of a data point traces a curve through the point.
If the encoder maps this curve to an axis-aligned line in the latent space and is locally invertible at the point, then we can interpret the component associated with this axis as the factor of variation underlying the data point.
Disentanglement is then a question of whether the data's probability density function factors along these curves.
We use tools from Riemannian geometry \cite{lee2018introduction} to make this picture precise.

In this paper, we introduce Riemannian ICA (RICA), a local geometric analogue of ICA that defines independent factors of variation from a data density and a Riemannian metric.
RICA is motivated by pretrained models that can be represented by Riemannian metrics, and interprets \emph{Riemann normal coordinates} \cite{lee2018introduction} as latent variables that are defined locally around a chosen base point.
This modeling choice preserves the local change-of-variables structure used by ICA, can be applied to non-invertible encoder models through the pullback metric, and gives a clear geometric interpretation of the factors of variation as geodesics through the base point.
Since the map to normal coordinates is invertible in a neighborhood of the base point, we are able to directly pushforward the data density through it and use the change of variables formula to obtain a probability density for the latent variables.
If this latent density factors into a product of independent distributions, then the latents have independent components.
However, it is possible that a pretrained model cannot disentangle data in any neighborhood of a point.
We therefore introduce \emph{pointwise disentanglement}, a necessary pointwise condition for disentanglement that captures second-order statistical independence at the point.
Our main theoretical contribution is a tensor that encodes pointwise disentanglement called the \emph{disentanglement tensor}.
This tensor depends on the data through the second derivatives of the log likelihood function and on the metric through Ricci curvature.
Our algorithm for pointwise disentanglement diagonalizes this tensor, which amounts to solving a generalized eigenvalue problem.
We validate this theory in a controlled source recovery setting where the ground-truth sources, metric, and density are known.
Across the tested manifolds and coordinate charts, RICA recovers the ground-truth sources, whereas the baselines are sensitive to the chart in which the observations are presented.
RICA recasts disentanglement as a local geometric property of a Riemannian metric relative to a data density, rather than as a property of a global invertible generative model.
This shift makes it possible to study disentanglement in pretrained representations without treating them as the latent variables of a generative model.

\section{Preliminaries}
\label{sec:preliminaries}
In this section, we review independent component analysis and introduce the necessary components of Riemannian geometry that we use in our main contributions.
We use bold symbols for vector-valued random variables and unbolded symbols for their components, so that $s_i$ denotes the $i$'th component of $\s$. We use unsubscripted unbolded symbols for functions, and a subscript $x$ for components of geometric objects expressed in data-space coordinates.
We denote data points by $\x \in \mathbb{R}^n$, the data density with respect to the Lebesgue measure by $p_x(\x)$, and latent variables, or \emph{sources}, by $\s \in \mathbb{R}^d$, with probability density $p_s(\s)$.
Each dimension $s_i \in \mathbb{R}$ of $\s$ is called a \emph{factor} or \emph{component}.
We use $f: \mathbb{R}^d \to \mathbb{R}^n$ to denote a decoder that maps latent variables to data, and $h: \mathbb{R}^n \to \mathbb{R}^d$ to denote an encoder that maps data to latents.

While there is no universal notion of disentanglement \cite{wang2024disentangled}, the one we use is closely related to ICA.
\begin{definition}[Disentangled encoder]
  \label{def:disentangled-encoder}
  An encoder $h \colon \R^n \to \R^d$ \emph{disentangles} the data density $p_x(\x)$ if the pushforward measure $p_\s = h_{\#} p_\x$ admits a density that factors across dimensions,
  \begin{equation}
    p_s(\s) = \prod_{i=1}^d p_{s,i}(s_i)
  \end{equation}
\end{definition}
\Cref{def:disentangled-encoder} specifies a property of an encoder that relates to a probability distribution, but does not say how to obtain a disentangled encoder.
ICA studies how to recover such a disentangled encoder from observations, which we review next.

\begin{figure}[t]
  \centering
  \begin{subfigure}[t]{0.32\textwidth}
    \centering
    \includegraphics[width=\textwidth]{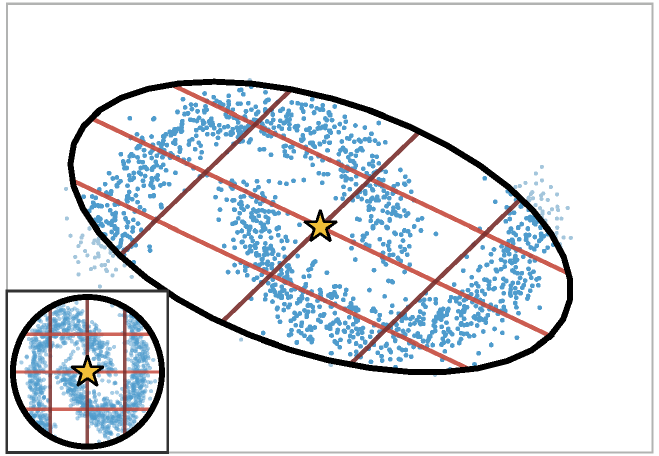}
    \caption{Model for ICA.}
    \label{fig:rica-overview-ica}
  \end{subfigure}\hfill
  \begin{subfigure}[t]{0.32\textwidth}
    \centering
    \includegraphics[width=\textwidth]{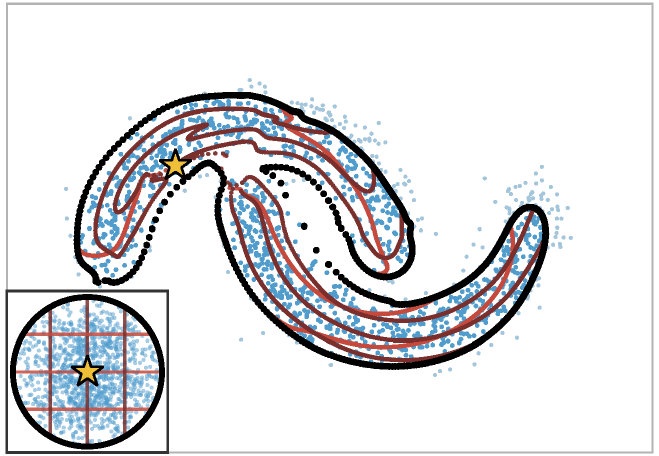}
    \caption{Nonlinear ICA without constraints \cite{dinh2017density}.}
    \label{fig:rica-overview-realnvp}
  \end{subfigure}\hfill
  \begin{subfigure}[t]{0.32\textwidth}
    \centering
    \includegraphics[width=\textwidth]{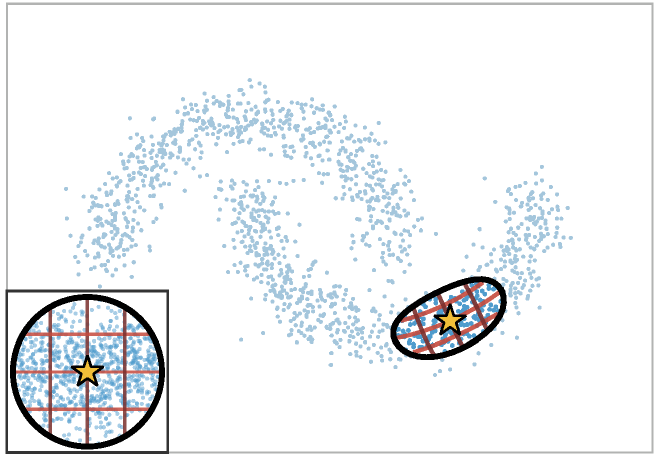}
    \caption{RICA with a smoothed Monge metric and uniform Hausdorff density.}
    \label{fig:rica-overview-rica}
  \end{subfigure}
  \caption{%
    Comparison of ICA, nonlinear ICA, and RICA.
    ICA fits a model that is too simple to capture the structure of the data.
    Nonlinear ICA is able to model the data but requires extra constraints to learn meaningful latent structure.
    RICA uses the local geometry encoded by the metric to find latent structure aligned with the data density.
    See \Cref{app:figure1-monge-metric} for details on the RICA panel.
  }
  \label{fig:rica-overview}
\end{figure}

\subsection{Independent component analysis}
\label{sec:ica}
Independent component analysis (ICA) \cite{comon1994independent,hyvarinen2000ica,hyvarinen2023nonlinear} is a classical algorithm for disentanglement that assumes data is generated by the model
\begin{align}
  \label{eq:ica-model}
  \x = J \s, \quad \s \sim \prod_{i=1}^n p_{s,i}(s_i)
\end{align}
where $J \in \mathbb{R}^{n \times n}$ is an invertible matrix and at most one of the source densities, $p_{s,i}$, is Gaussian.
Under these assumptions, the sources can be identified from data up to permutations and scaling \cite{comon1994independent}.
ICA has had success in problems where data actually satisfies these assumptions, such as the cocktail party problem \cite{hyvarinen2000ica} where observations are generated by a linear mixture of independent audio signals.
However, this linear model is unrealistic for high-dimensional datasets encountered in practice, such as ImageNet \cite{imagenet}, which bars the interpretation of the sources recovered by ICA as factors of variation underlying the data (\Cref{fig:rica-overview-ica}).

To overcome this limitation, nonlinear ICA \cite{nonlinear_ica,hyvarinen2019nonlinear,hyvarinen2023nonlinear} was originally proposed as a more expressive generalization of ICA.
Its model is $\x = f(\s), \;\s \sim p_s(\s)$ where $f: \mathbb{R}^n \to \mathbb{R}^n$ is an invertible function, often parameterized by a normalizing flow \cite{pmlr-v37-rezende15,papamakarios2021normalizing}, and $\s$ has independent components.
The model of nonlinear ICA is flexible enough to represent any non-degenerate distribution of data \cite{papamakarios2021normalizing}, but is not identifiable without extra constraints because infinitely many latent parameterizations can produce the same data distribution \cite{nonlinear_ica,moser1965volume,buchholz2022function}.
This leads to models like the one shown in \Cref{fig:rica-overview-realnvp} that are able to model the data's distribution well but cannot recover meaningful latent structure.
Several constraints have been proposed in pursuit of identifiability, such as incorporating auxiliary variables \cite{hyvarinen2019nonlinear,horan2021unsupervised,zheng2022identifiability,lachapelle2022disentanglement} or making structural assumptions about the decoder \cite{buchholz2022function,buchholz2023learning,gresele2021independent,horan2021unsupervised}.
Despite a push to use nonlinear ICA for disentanglement \cite{hyvarinen2023nonlinear}, research in this area shifted toward other approaches in recent years \cite{scholkopf2021toward,buchholz2023learning,Brady2024interaction,matthes2026mechanistic,rajendran2024causal,rusak2024contrastive}.

A core limitation of ICA and its variants is the assumption that the decoder is globally invertible.
This assumption makes ICA and nonlinear ICA fundamentally misspecified when data has a non-trivial topology \cite{cornish2020relaxing}, and excludes non-invertible neural network architectures as encoders.
There are two key reasons why this restrictive assumption has not been relaxed.
The first reason is the problem setting.
Disentangling the factors of variation underlying data implies both the ability to recover sources from data and the ability to vary data from sources because we must know exactly how the data varies.
This is not possible without invertibility unless we take a fundamentally different approach, such as employing ELBO-based methods like the $\beta$-VAE \cite{higgins_beta-vae_2017,reizinger2022embrace}.
The second reason is the change of variables formula, which relates the density function of data to the density function of the latents.
Under an invertible decoder where $h = f^{-1}$, the change of variables formula is given by
\begin{align}
  \label{eq:change-of-variables-formula}
  \log p_s(\s) = \log p_x(\x) + \log|\det J(\x)|
  \quad \text{ where } \s = h(\x) \text{ and } J(\x) = \frac{df(\s)}{d\s}
\end{align}
If $h$ is not invertible, then the formula no longer holds, making the analysis of $p_s(\s)$ considerably more difficult \cite{balestriero2025gaussian}.
As a result, non-invertible pretrained encoders lack the generative structure needed to analyze whether their features define independent factors of variation.

Our main insight is that disentanglement is fundamentally a local question about how an individual data point varies with respect to latent variations, so global invertibility is not actually required (\Cref{fig:rica-overview-rica}).
The intuitive picture we adopt is to think about how position on a curve in the data space corresponds to an axis-aligned variation in the latent space.
Suppose we have a base point $\x_0 \in \mathbb{R}^n$ for which we compute its sources $\s_0 = h(\x_0)$, and suppose that we vary the $i$'th source a small amount $\Delta s_i$, $\s = \s_0 + \Delta s_i \e_i$, where $\e_i$ is the $i$'th standard basis vector.
Because $h$ may not be invertible, we cannot directly compute the new data point $\x = h^{-1}(\s)$.
What we can do instead is search for a curve $\x(t)$ that starts at $\x_0$ and ends at some $\x$ so that the latent curve, $\s(t) = h(\x(t))$, is a straight line that starts at $\s_0$ and ends at $\s$.
However, the standard ICA toolkit does not contain the necessary tools or even language for formalizing this.
Fortunately, Riemannian geometry does.

\subsection{Riemannian geometry and normal coordinates}
\label{sec:riemannian-geometry}
We review the necessary tools from Riemannian geometry following \cite{Lee00,lee2018introduction}, but adopt a more suggestive notation to make the connections to ICA as direct as possible.
See \Cref{app:geometry-background} for a more detailed review.
We emphasize that every geometric object we introduce in this paper is parameterized and computed in the data space.

A Riemannian manifold $(\M, g)$ is a smooth manifold $\M$ equipped with a positive-definite inner product $g$, called the \emph{Riemannian metric}.
We focus on the setting where $\M = \R^n$ because data is usually collected in a Euclidean space, but our approach is directly applicable to general smooth manifolds.
In this setting, the metric is represented by a symmetric, positive-definite matrix $g_x(\x) \in \mathbb{R}^{n\times n}$.
It is used to define a pointwise inner product between two tangent vectors $u$ and $v$ at some $\x \in \R^n$, whose components are represented by $u_x, v_x \in \R^n$, respectively.
The inner product of $u$ and $v$ with respect to $g$ is given by $g(u,v) = u_x^\top g_x(\x)\, v_x$.

$g$, $u$, and $v$ are \emph{tensors}, meaning that they exist independently of any coordinate system.
In particular, $u$ and $v$ are $(1,0)$-tensors, while $g$ is a $(0,2)$-tensor.
These tensor types determine how their coordinate representations transform under a change of coordinates in order to preserve the meaning of the underlying objects.
If $\phi \colon \mathbb{R}^n \to \mathbb{R}^n$ is a smooth, invertible function where $\x = \phi(\z)$, then the components of the tangent vectors and metric in $z$-coordinates are related by
\begin{equation}
  \label{eq:metric-transformation-rule}
  u_z =
  \frac{\partial \phi(\z)}{\partial \z}^{-1}
  u_x, \quad
  v_z =
  \frac{\partial \phi(\z)}{\partial \z}^{-1}
  v_x, \quad
  g_z(\z) =
  \frac{\partial \phi(\z)}{\partial \z}^{\top}
  g_x(\x)
  \frac{\partial \phi(\z)}{\partial \z}
\end{equation}
These transformation rules preserve the inner product $g(u,v) = u_x^\top g_x(\x)v_x = u_z^\top g_z(\z)v_z$.
Additionally, this inner product induces a notion of \emph{$g$-orthogonality}.
Suppose $J_x \in \mathbb{R}^{n \times n}$ is a matrix whose columns represent components of $n$ tangent vectors.
Then $J_x$ is $g$-orthogonal if $J_x^\top g_x(\x)\, J_x = I_n$.

A Riemannian metric determines distances on the manifold.
Let $\x_0, \x \in \mathbb{R}^n$ be two points.
If the points are sufficiently close, then there exists a unique shortest path between them called a \emph{geodesic}.
Geodesics are completely characterized by their base point $\x_0$ and an initial velocity $v_x \in \mathbb{R}^n$.
The map that sends the initial velocity to the point reached by the geodesic after $1$ unit of time is called the \emph{exponential map}, and is denoted by $\x = \exp_{\x_0}(v_x)$.
In a sufficiently small neighborhood of $\x_0$, the exponential map is invertible, and its inverse is called the \emph{logarithmic map}, $\log_{\x_0}(\x)$, which recovers the initial velocity from an endpoint $\x$.
$\log_{\x_0}(\x)$ can be computed directly from the metric $g_x$, but is computationally expensive in general.
Under the Euclidean metric, the logarithmic map is simply $\log_{\x_0}(\x) = \x - \x_0$.

The preceding machinery leads to \emph{Riemann normal coordinates}, which are an orthonormal coordinate system at $\x_0$.
Normal coordinates are defined on a normal neighborhood of $\x_0$, where the exponential map is invertible \cite{lee2018introduction}.
They are unique up to the choice of a $g$-orthogonal matrix that orients the coordinate system.
Let $J_x \in \mathbb{R}^{n \times n}$ denote this orientation matrix.
Then the normal coordinate map, also called the \emph{normal coordinate chart}, is given by
\begin{align}
  \label{eq:normal-coordinate-chart}
  h_{\scriptstyle \x_0, J_x}(\x) := J_x^{-1}\log_{\x_0}(\x)
\end{align}
We denote these normal coordinates by $\s = h_{\scriptstyle \x_0, J_x}(\x)$.
At the base point, the normal coordinate chart satisfies $\zero = h_{\scriptstyle \x_0, J_x}(\x_0)$ and $\frac{\partial}{\partial \s}h_{\scriptstyle \x_0, J_x}^{-1}(\zero) = J_x$.

A useful property of normal coordinates is that the metric is Euclidean at the origin and has zero first derivatives there.
Let $g_s(\zero) = J_x^\top g_x(\x_0) J_x$ be the components of the metric in normal coordinates at the origin.
Then the value, gradient, and Hessian of the log determinant of the metric are given by
\begin{align}
  \label{eq:normal-coordinate-metric}
  \frac{1}{2}\log \det g_s(\zero) = 0,
  \quad
  \frac{\partial \frac{1}{2} \log \det g_s(\zero)}{\partial \s} = \zero,
  \quad
  \frac{\partial^2 \frac{1}{2} \log \det g_s(\zero)}{\partial \s \partial \s^\top} = -\frac{1}{3}\Ric_s(\zero)
\end{align}
where $\Ric_s(\zero) \in \mathbb{R}^{n \times n}$ is the Ricci curvature tensor in normal coordinates at the origin.
The Ricci tensor is a $(0,2)$-tensor, and so it transforms according to \Cref{eq:metric-transformation-rule}.
For simplicity, we use \Cref{eq:normal-coordinate-metric} as an operational definition of Ricci curvature.
In our main contributions, Ricci curvature appears as a correction to the second derivative of the change of variables formula that is needed on a curved manifold.

Lastly, the Hausdorff density \cite{williams2025geodesic} $\rho_x(\x) = p_x(\x) / \sqrt{\det g_x(\x)}$ is an intrinsic density function that is defined with respect to the Riemannian volume element $dV_g = \sqrt{\det g_x(\x)}\, d\x$.
The Hausdorff density gives a coordinate-independent way to express probability, which we use to derive a change of variables formula on a Riemannian manifold.
See \Cref{app:sec:hausdorff-density} for more details.

\section{Method}
\label{sec:method}

In this section, we present our contributions.
Our method takes as input a probability density function, $p_x(\x)$, and a Riemannian metric $g_x(\x)$.
There are many ways to construct Riemannian metrics \cite{gruffaz2025riemannian}.
For example, metrics can be constructed from point clouds using diffusion geometry \cite{jones2026computing,bamberger2026riemannian}, from probability distributions \cite{hartmann2022lagrangian,williams2025geodesic}, or from generative models \cite{arvanitidis2018latent,dombrowski2021diffeomorphic}.
Our method is motivated by the setting where one is given a pretrained encoder, $h \colon \R^n \to \R^d$, for which the pullback metric is a natural choice.

Suppose $h \colon \R^n \to \R^d$ is a smooth, pretrained encoder and, for the moment, assume that $d = n$.
The pullback metric of the Euclidean metric on $\R^d$ through $h$ is given by $g_x(\x) = \scriptstyle \tfrac{\partial h(\x)}{\partial \x}^\top \tfrac{\partial h(\x)}{\partial \x}$.
When $h$ is locally invertible, this metric measures the length of a path in data space by the Euclidean length of its image under $h$.
This is consistent with the intuitive picture of factors of variation at the end of \Cref{sec:ica} because local geodesic variations under the pullback metric correspond to linear changes in the latent space.
When $d \neq n$, the metric may be degenerate, so we adopt a simple heuristic to ensure that it has full rank.
We use the metric $g_x(\x) = (1 - \lambda)I_n + \lambda \scriptstyle \frac{\partial h(\x)}{\partial \x}^\top \frac{\partial h(\x)}{\partial \x}$, where $\lambda \in (0, 1)$, to blend the Euclidean metric with the pullback metric.
In the rest of this section, we introduce our contributions using an abstract Riemannian metric.

\subsection{Normal coordinates are adaptive latent variables}
\label{sec:factors-of-variation}

Recall that ICA requires an invertible encoder $h(\x) \mapsto \s$.
Our first contribution is to interpret the normal coordinate chart from \Cref{eq:normal-coordinate-chart} as a \emph{local} encoder for ICA.
Let $h_{\scriptstyle \x_0, J_x}$ be the normal coordinate chart centered at $\x_0$ and oriented by $J_x$.
Because this chart maps the base point to the origin, $h_{\scriptstyle \x_0, J_x}(\x_0) = \zero$, varying the $i$'th component by $\Delta s_i$ gives $\s = \Delta s_i \e_i$, where $\e_i$ is the $i$'th standard basis vector.
The corresponding data variation follows the radial geodesic $\x = \exp_{\x_0}(J_x \s)$.
Thus, the factors of variation through $\x_0$ are geodesics under the metric, formalizing the intuitive picture we described at the end of \Cref{sec:ica}.

Next, we need to understand how probability density fits into the picture.
$h_{\scriptstyle \x_0, J_x}$ is invertible in a neighborhood of $\x_0$, so we can pushforward the restriction of the data density to this neighborhood and apply the change of variables formula.
Therefore, by \Cref{def:disentangled-encoder}, a normal coordinate chart disentangles the data near $\x_0$ when the pushed-forward density factors across its coordinates.
This lets us define what it means for a metric to disentangle data on a neighborhood of $\x_0$.
However, such a disentangling neighborhood may not exist \cite{krishnan2021diagonalizing,tod1992choosing}.
Since we do not know which models fail to disentangle data a priori, we propose to check a necessary condition for disentanglement that we call \emph{pointwise disentanglement}.

\subsection{Pointwise disentanglement}
\label{sec:pointwise-disentanglement}
Pointwise disentanglement examines whether a normal coordinate chart disentangles data directly at its base point.
To formalize this, we first need to express the change of variables formula at the base point.
\begin{restatable}[Pointwise change of variables formula]{proposition}{PointwiseChangeOfVariablesFormula}
  \label{prop:pointwise-change-of-variables-formula}
  Let $g_x$ be a Riemannian metric on $\mathbb{R}^n$, $\x_0 \in \mathbb{R}^n$ be the base point, $J_x \in \mathbb{R}^{n \times n}$ be a $g$-orthogonal matrix, and $h_{\scriptstyle \x_0, J_x}$ be the normal coordinate chart.
  At the origin, the change of variables formula is given by
  \begin{equation}
    \label{eq:normal-change-of-variables-formula}
    \log p_s(\zero) = \log p_x(\x_0) + \log|\det J_x|
    \quad \text{ where } \zero = h_{\scriptstyle \x_0, J_x}(\x_0)
  \end{equation}
  Additionally, the Hausdorff density at the base point equals the latent density, $\rho_x(\x_0) = p_s(\zero)$.
\end{restatable}
See \Cref{app:prop:pointwise-change-of-variables-formula} for the derivation.
This expression is consistent with the standard change of variables formula in \Cref{eq:change-of-variables-formula} because $J_x$ is the Jacobian matrix of the map from latents to data, $h_{\scriptstyle \x_0, J_x}^{-1}(\s) \mapsto \x$, at $\s=\zero$.
\Cref{prop:pointwise-change-of-variables-formula} lets us define pointwise disentanglement.
\begin{restatable}[Pointwise disentanglement]{definition}{LogHessianIndependenceProp}
  \label{def:local-independence}
  Let $\x_0 \in \mathbb{R}^n$ be a base point, $J_x \in \mathbb{R}^{n \times n}$ be a $g$-orthogonal matrix, and let $h_{\scriptstyle \x_0, J_x}$ be the corresponding normal coordinate chart.
  $h_{\scriptstyle \x_0, J_x}$ pointwise disentangles the data distribution $p_x$ at the base point $\x_0$ if the latent log-likelihood has a diagonal Hessian at the origin:
  \begin{equation}
    \label{eq:local-independence}
    \frac{\partial^2 \log p_s(\zero)}{\partial \s \partial \s^\top} \text{ is diagonal.}
  \end{equation}
\end{restatable}
It is easy to confirm that this condition follows from statistical independence by differentiating the latent log-likelihood twice.
This definition is closely related to \Cref{def:disentangled-encoder}, except that pointwise disentanglement uses the normal coordinate chart in place of an explicit encoder, and the independence statement is only made to second order at a point instead of globally.
Like \Cref{def:disentangled-encoder}, this definition says nothing about how to choose $J_x$ to make the Hessian diagonal.
The next subsection shows how to choose $J_x$ so that the recovered latents are pointwise disentangled.

\subsection{Disentanglement tensor and Riemannian ICA}
\label{sec:disentanglement-tensor}
Next, we present our primary contribution, which is a tensor that captures pointwise disentanglement.
We call this tensor the \emph{disentanglement tensor} because its diagonalization identifies a normal coordinate chart whose latents are pointwise disentangled.

\begin{figure}[t]
  \centering
  \begin{subfigure}[t]{0.50\textwidth}
    \centering
    \includegraphics[width=\linewidth]{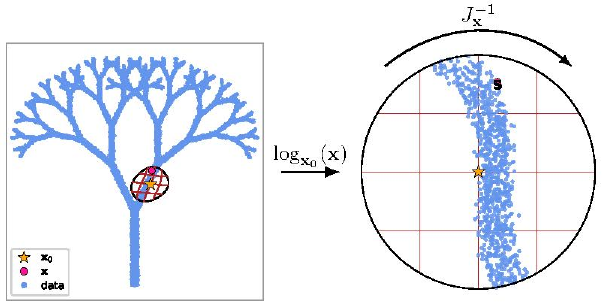}
    \subcaption{RICA consists of a log map followed by a $g$-rotation.}
    \label{fig:log-map-overview-chart}
  \end{subfigure}\hfill
  \begin{subfigure}[t]{0.48\textwidth}
    \centering
    \includegraphics[width=\textwidth]{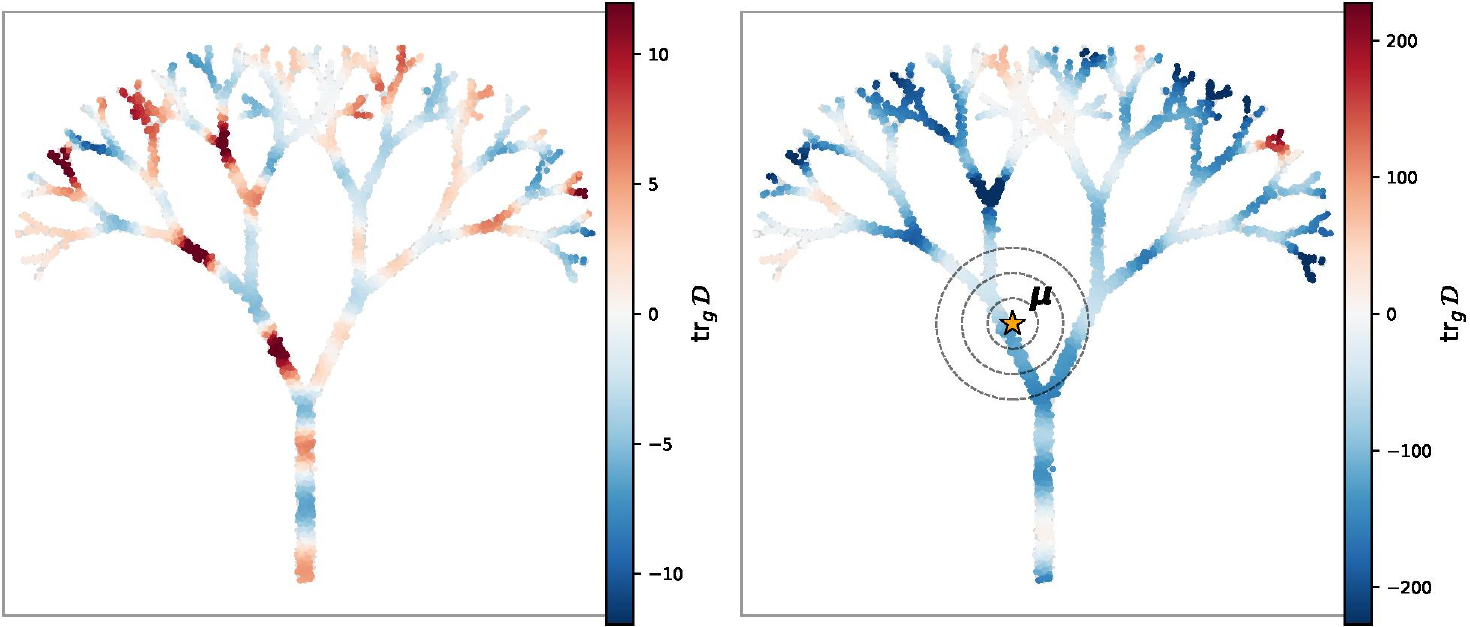}
    \subcaption{$\operatorname{tr}_{g} \mathcal{D}$ for different data densities.}
    \label{fig:log-map-overview-trace}
  \end{subfigure}
  \caption{%
    The blue points are a point cloud from which we construct a diffusion-geometry metric \cite{jones2026computing}.
    Plot (a) illustrates RICA for a Gaussian kernel density estimate of the point cloud.
    At a selected base point $\x_0$, latents are computed as $\s = J_x^{-1} \log_{\x_0}(\x)$.
    Different choices of $J_x$ rotate the latents, and selecting $J_x$ with \Cref{thm:rica} aligns the latents with the coordinate axes to second order.
    Plot (b) shows the trace of the disentanglement tensor $\operatorname{tr}_{g} \mathcal{D}$ for the uniform Hausdorff density on the left and the isotropic Gaussian density on the right.
  }
  \label{fig:log-map-overview}
\end{figure}

\begin{restatable}[Disentanglement tensor]{theorem}{DisentanglementTensor}
  \label{thm:disentanglement-tensor}
  Let $p_x$ and $g_x$ be the probability density function of data and the metric on the data space, respectively, and let $\rho$ denote the Hausdorff density.
  The \emph{disentanglement tensor} is defined as
  \begin{equation}
    \label{eq:disentanglement-tensor}
    \mathcal{D} = \nabla^2 \log \rho - \tfrac{1}{3}\Ric
  \end{equation}
  where $\nabla^2$ is the covariant Hessian under the Levi-Civita connection and $\Ric$ is the Ricci curvature tensor.
  For any base point $\x_0 \in \mathbb{R}^n$ and any $g$-orthogonal matrix $J_x$, let $h_{\scriptstyle \x_0, J_x}$ be the corresponding normal coordinate chart and let $p_s = (h_{\scriptstyle \x_0, J_x})_{\#} p_x$ be the induced latent density on this normal coordinate neighborhood.
  Then the Hessian of the normal log-likelihood is given by
  \begin{equation}
    \label{eq:hessian-normal-log-likelihood}
    \frac{\partial^2 \log p_s(\zero)}{\partial \s \partial \s^\top} = \mathcal{D}_s(\zero)
  \end{equation}
  where $\mathcal{D}_s(\zero)$ is the matrix containing the components of the disentanglement tensor in the normal coordinate system.
\end{restatable}
The derivation in \Cref{app:thm:disentanglement-tensor} follows by differentiating \Cref{eq:normal-change-of-variables-formula} twice and identifying the resulting tensorial terms.
In practice, the disentanglement tensor can be computed in whichever coordinate system is most convenient and then transformed to normal coordinates using the $(0,2)$ tensor transformation rule.
For example, in data-space coordinates, this computation is given by the following proposition.

\begin{restatable}[Coordinate formula for the disentanglement tensor]{proposition}{DisentanglementTensorCoordinates}
  \label{prop:disentanglement-tensor-coordinates}
  Let $p_x$ and $g_x$ be the probability density function and metric on data space, respectively.
  Let $\Gamma_x(\x_0)^k_{ij}$ denote the Christoffel symbols of the Levi-Civita connection induced by $g_x$ at $\x_0$, which are defined by
  \begin{equation}
    \Gamma_x(\x_0)^k_{ij}
    = \sum_{m=1}^n \frac{1}{2} g_x(\x_0)^{km}
    \left(
      \frac{\partial g_x(\x_0)_{im}}{\partial x^j}
      + \frac{\partial g_x(\x_0)_{jm}}{\partial x^i}
      - \frac{\partial g_x(\x_0)_{ij}}{\partial x^m}
    \right)
  \end{equation}
  where $g_x(\x_0)^{km}$ denotes the $(k,m)$ index of the inverse matrix $g_x(\x_0)^{-1}$.
  The $(i,j)$ component of the disentanglement tensor in data space coordinates is given by
  {\small
  \begin{equation}
    \label{eq:disentanglement-tensor-data-coordinates}
    \mathcal{D}_{x}(\x_0)_{ij}
    =
    \frac{\partial^2 r(\x_0)}{\partial x^i \partial x^j}
    -
    \sum_{k=1}^n \Gamma_x(\x_0)^k_{ij}
    \frac{\partial r(\x_0)}{\partial x^k}
    -
    \frac{1}{3}\sum_{a=1}^n \frac{\partial \Gamma_x(\x_0)^a_{ij}}{\partial x^a}
    +
    \frac{1}{3}\sum_{a=1}^n \sum_{b=1}^n \Gamma_x(\x_0)^a_{ib}\Gamma_x(\x_0)^b_{aj}
  \end{equation}
  }
  where $r(\x_0) = \log p_x(\x_0) - \frac{1}{3}\log\det g_x(\x_0)$.
  Additionally, if $J_x$ is a $g$-orthogonal matrix, then the components of the disentanglement tensor in normal coordinates $\s = h_{\scriptstyle \x_0, J_x}(\x)$ are given by
  \begin{equation}
    \label{eq:disentanglement-tensor-normal-coordinates}
    \mathcal{D}_s(\zero) = {J_x}^\top \mathcal{D}_{x}(\x_0) {J_x}
  \end{equation}
\end{restatable}
See \Cref{app:prop:disentanglement-tensor-coordinates} for the derivation.
This coordinate formula makes $\mathcal{D}$ computable, but also shows why exact computation can be expensive.
In closed form, it requires second order derivatives of the metric and the log likelihood, which are not always tractable in practice.
We discuss these practical limitations in \Cref{sec:practical-limitations}.
When $\mathcal{D}_x(\x_0)$ is available, pointwise disentanglement reduces to choosing $J_x$ so that \Cref{eq:disentanglement-tensor-normal-coordinates} yields a diagonal matrix.
This gives our algorithm for performing pointwise disentanglement.

\begin{restatable}[RICA]{proposition}{PointwiseDisentanglementRICA}
  \label{thm:rica}
  Let $\mathcal{D}_x(\x_0) \in \mathbb{R}^{n \times n}$ be the coordinate matrix of the disentanglement tensor at the base point $\x_0$.
  Since $\mathcal{D}_x(\x_0)$ is symmetric and $g_x(\x_0)$ is positive definite, there exists a $g$-orthogonal matrix $J^*_x$ solving
  \begin{equation}
    {J^*_x}^\top \mathcal{D}_x(\x_0)\, J^*_x = \Lambda,
    \quad
    {J^*_x}^\top g_x(\x_0)\, J^*_x = I_n
  \end{equation}
  where $\Lambda$ is diagonal.
  Then $\s = h_{\scriptstyle \x_0, J^*_x}(\x)$ is pointwise disentangled at the base point, and is identifiable up to sign flips and permutations if the entries of $\Lambda$ are distinct.
\end{restatable}
See \Cref{app:thm:rica} for the proof.
\Cref{thm:rica} gives an algorithm for performing pointwise disentanglement that is equivalent to solving a generalized eigenvalue problem, which differs from standard eigenvalue problems only in how orthogonality is defined \cite{golub2013matrix}.
As a consequence, the uniqueness of the recovered latents depends on the uniqueness of the eigendecomposition of $\mathcal{D}_x(\x_0)$.

We can visualize the algorithm for pointwise disentanglement in \Cref{fig:log-map-overview}.
The blue dots represent a point cloud from which we construct a diffusion-geometry metric \cite{jones2026computing}.
\Cref{fig:log-map-overview-chart} illustrates how this metric disentangles the kernel density estimate of the point cloud using RICA.
Pointwise disentanglement is performed by computing the disentanglement tensor at the base point using \Cref{eq:disentanglement-tensor-data-coordinates}, finding the $g$-rotation $J_x^*$ that diagonalizes it via \Cref{thm:rica}, and then computing the latents as $\s = {J^*_x}^{-1} \log_{\x_0}(\x)$.
The plots in \Cref{fig:log-map-overview-trace} illustrate how changing the density affects the disentanglement tensor.
The left panel of \Cref{fig:log-map-overview-trace} uses the uniform Hausdorff density and the right panel of \Cref{fig:log-map-overview-trace} uses an isotropic Gaussian.
The magnitude of the disentanglement tensor increases significantly when we switch from the uniform Hausdorff density to the isotropic Gaussian.

\subsection{Practical limitations}
\label{sec:practical-limitations}
RICA provides a theoretically justified algorithm for performing pointwise disentanglement with metrics induced by encoders such as those based on self-supervised learning \cite{radford2021learning,caron2021emerging,oquab2024dinov,simeoni2025dinov3}.
However, its computational cost is prohibitive for typical pretrained models encountered in practice because, under the pullback metric, computing \Cref{eq:disentanglement-tensor-data-coordinates} requires third-order automatic differentiation through the encoder.
Additionally, RICA requires the ability to compute the log likelihood of the data density in closed form, which is not always available.
We leave the development of a scalable empirical algorithm for future work.

\section{Experiments}
\label{sec:experiments}

We validate our theory in a source recovery setting where we generate data from ground-truth sources and test the ability of RICA and the baselines to recover the sources.

\subsection{Source recovery}
\label{sec:source-recovery}

We evaluate whether RICA can recover known sources better than ICA and a nonlinear ICA baseline.
This experiment is designed as a controlled validation of the theory.
Complete experimental details are given in \Cref{app:source-recovery-experimental-details}.
We evaluate four tractable geometries: the sphere, torus, hyperbolic space, and SPD matrices.
To test the role that the coordinates used to represent the observations plays, we generate observations for each geometry in both an intrinsic coordinate chart and an alternate coordinate chart, which are specified in \Cref{app:source-recovery-experimental-details:manifolds-and-charts}.
For each geometry, we fix a base point $\x_0$ and randomly draw a $g$-orthonormal matrix $J_x$.
We then generate $N$ paired samples $(\x^{(\ell)}, \s^{(\ell)})_{\ell=1}^N$ from the following generative model:
\begin{equation}
\label{eq:tangent-source-gen-model}
\begin{aligned}
  \x &= \exp_{\x_0}(J_x \s), \quad \text{where} \quad s_i \sim \operatorname{Logistic}(0, b_i), \quad b_i = b\, r_s^{\,i-1}, \quad \text{for } i = 1, \ldots, n
\end{aligned}
\end{equation}
where $\x_0 \in \R^n$ is a fixed base point, $J_x$ is $g$-orthonormal at $\x_0$, and the scalars $b, r_s > 0$ control the overall spread of $\s$ and the anisotropy of the sources, respectively.
For these source-recovery experiments, the problem scale is small enough that the disentanglement tensor can be computed exactly using \Cref{eq:disentanglement-tensor-data-coordinates}.

Each method we evaluate recovers sources $\hat{\s}^{(\ell)} \in \R^n$ from observations $\x^{(\ell)}$ from \Cref{eq:tangent-source-gen-model}.
We evaluate these sources using the mean correlation coefficient (MCC) \cite{khemakhem2020variational} and total correlation (TC) \cite{5392532}.
\begin{equation}
  \label{eq:mcc-multi-info}
  \mathrm{MCC}\bigl(\s, \hat{\s}\bigr) \;=\; \frac{1}{n}\, \max_{\pi \in S_n} \sum_{i=1}^n \bigl|\mathrm{corr}\bigl(\hat s_i,\, s_{\pi(i)}\bigr)\bigr|,\qquad \mathrm{TC}(\hat\s) \;=\;\mathrm{KL}\Bigl[p_{\hat{s}}(\hat\s) \,\Big\|\, \textstyle\prod_{i=1}^n p_{\hat s,i}(\hat s_i)\Bigr]
\end{equation}
where $S_n$ is the set of permutations of $\{1, \ldots, n\}$, the maximum is computed by linear assignment on the absolute-correlation matrix, and $\mathrm{corr}$ is the sample Pearson correlation across the $N$ paired observations.
MCC equals $1$ when the recovered sources match the ground truth up to signed permutation.
TC measures statistical dependence between recovered latents and vanishes when the latents have independent components.
We estimate $\mathrm{TC}$ with the nearest-neighbor estimator of \cite{kraskov2004estimating} at $k = 3$ neighbors.

The methods we use to estimate the sources from observations are RICA \Cref{thm:rica}, FastICA \cite{fastica}, and conformal nonlinear ICA \cite{buchholz2022function} (NL-ICA).
For RICA, we diagonalize the disentanglement tensor in closed form to obtain an orientation $J_x^*$ at the base point, and then recover sources from the rotated normal coordinate chart, $\hat\s = (J_x^*)^{-1} \log_{\x_0}(\x)$.
For FastICA and conformal nonlinear ICA, we recover sources from the latent spaces of their models.
Because MCC and TC are invariant to shifts and scale changes, we do not offset the baseline latents to place the base point at the origin.

We use a conformal model for the nonlinear ICA baseline because it has identifiability guarantees \cite{buchholz2022function}.
This model has the property that the Jacobian matrix of the decoder factors into the product of a scalar function and an orthogonal matrix.
In dimensions greater than two, Liouville's theorem restricts conformal maps to \emph{M\"obius transformations} \cite{buchholz2022function}.
We learn the parameters of a composition of $8$ M\"obius transformations by maximum likelihood under a logistic prior whose scales are matched to the source scales of the generative model.
See \Cref{app:source-recovery-experimental-details} for details.

\subsubsection{Results}
\label{sec:source-recovery-results}

\begin{wrapfigure}[20]{r}{0.44\linewidth}
  \centering
  \includegraphics[width=\linewidth]{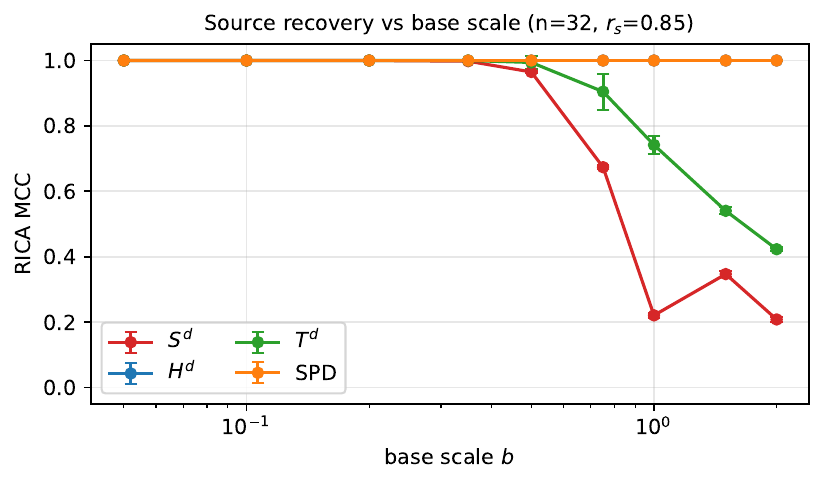}
  \caption{RICA MCC versus base scale $b$ across four geometries at target intrinsic dimension $n \approx 32$.
    Sphere, hyperbolic space, and torus use $n=32$, while SPD uses $8 \times 8$ positive-definite matrices, which have intrinsic dimension $n=8(8+1)/2=36$.
    RICA cannot recover the sources when samples are generated outside the injectivity radius of the exponential map.}
  \label{fig:source-recovery-b-sweep}
\end{wrapfigure}
\Cref{tab:source-recovery} reports MCC and TC for the sources recovered by RICA, ICA, and conformal nonlinear ICA across the different geometries and coordinate charts.
RICA recovers the sources faithfully across all geometries, while the performance of the ICA methods depends on the coordinate chart used to represent the observations.
This coordinate invariance is expected by construction.

Overall, FastICA outperforms NL-ICA in nearly all cases.
FastICA performs well because the samples stay near the base point, where the map from sources to observations is close to linear.
NL-ICA performs poorly because its maximum-likelihood objective is ill-conditioned.
The source scales $b_i$ decay geometrically, so they span many orders of magnitude.
Training then recovers the large-scale sources but loses the small-scale ones, and the number it loses grows with the dimension.
The appendix repeats the experiment at $n \approx 64$ (\Cref{tab:source-recovery-n64}), where the NL-ICA MCC falls to roughly $0.5$ on the intrinsic chart of every geometry.

\begin{table}[t]
\centering
\resizebox{\textwidth}{!}{%
\begin{tabular}{l l r r @{\hspace{1.5em}} r r @{\hspace{1.5em}} r r}
\toprule
 & & \multicolumn{2}{c}{RICA} & \multicolumn{2}{c}{FastICA} & \multicolumn{2}{c}{NL-ICA} \\
\cmidrule(lr){3-4} \cmidrule(lr){5-6} \cmidrule(lr){7-8}
Manifold & Chart & MCC $\uparrow$ & TC $\downarrow$ & MCC $\uparrow$ & TC $\downarrow$ & MCC $\uparrow$ & TC $\downarrow$ \\
\midrule
Sphere $S^d$ & ambient & \textbf{1.00 $\pm$ 0.00} & \textbf{0.01 $\pm$ 0.04} & 0.87 $\pm$ 0.02 & 0.17 $\pm$ 0.05 & 0.88 $\pm$ 0.01 & 0.95 $\pm$ 0.17 \\
 & stereographic & \textbf{1.00 $\pm$ 0.00} & \textbf{0.01 $\pm$ 0.04} & 0.95 $\pm$ 0.02 & 0.41 $\pm$ 0.04 & 0.81 $\pm$ 0.10 & 1.85 $\pm$ 0.44 \\
\midrule
Hyperbolic $H^d$ & Lorentz & \textbf{1.00 $\pm$ 0.00} & \textbf{0.01 $\pm$ 0.04} & 0.94 $\pm$ 0.00 & 0.71 $\pm$ 0.04 & 0.84 $\pm$ 0.01 & 1.80 $\pm$ 0.16 \\
 & Poincar\'e & \textbf{1.00 $\pm$ 0.00} & \textbf{0.01 $\pm$ 0.04} & 0.98 $\pm$ 0.01 & -0.18 $\pm$ 0.03 & 0.90 $\pm$ 0.02 & 0.86 $\pm$ 0.17 \\
\midrule
Flat torus $T^d$ & angle & \textbf{1.00 $\pm$ 0.00} & \textbf{0.01 $\pm$ 0.04} & 0.99 $\pm$ 0.00 & -0.02 $\pm$ 0.04 & 0.89 $\pm$ 0.02 & 1.19 $\pm$ 0.19 \\
 & sincos & \textbf{1.00 $\pm$ 0.00} & \textbf{0.01 $\pm$ 0.04} & 0.54 $\pm$ 0.02 & 3.09 $\pm$ 0.13 & 0.35 $\pm$ 0.01 & 7.46 $\pm$ 0.40 \\
\midrule
SPD($p$) & log-Euclidean & \textbf{1.00 $\pm$ 0.00} & \textbf{-0.01 $\pm$ 0.04} & 0.99 $\pm$ 0.00 & -0.04 $\pm$ 0.04 & 0.85 $\pm$ 0.02 & 1.88 $\pm$ 0.26 \\
 & vech & \textbf{1.00 $\pm$ 0.00} & \textbf{-0.01 $\pm$ 0.04} & 0.38 $\pm$ 0.02 & 4.09 $\pm$ 0.07 & 0.51 $\pm$ 0.01 & 5.93 $\pm$ 0.11 \\
\bottomrule
\end{tabular}
}
\vspace{4pt}
\caption{Source recovery at target intrinsic dimension $n \approx 32$.
  Sphere, hyperbolic space, and torus use $n=32$, while SPD uses $8 \times 8$ positive-definite matrices, which have intrinsic dimension $n=8(8+1)/2=36$.
  We use $b=0.3$ and $r_s=0.85$, and average over 5 seeds and 5 random $g$-orthonormal matrices at the base point.
  RICA recovers the ground truth sources across every chart, while the baselines are sensitive to the chart.}
\label{tab:source-recovery}
\end{table}

\begin{figure}[t]
\centering
\includegraphics[width=\textwidth]{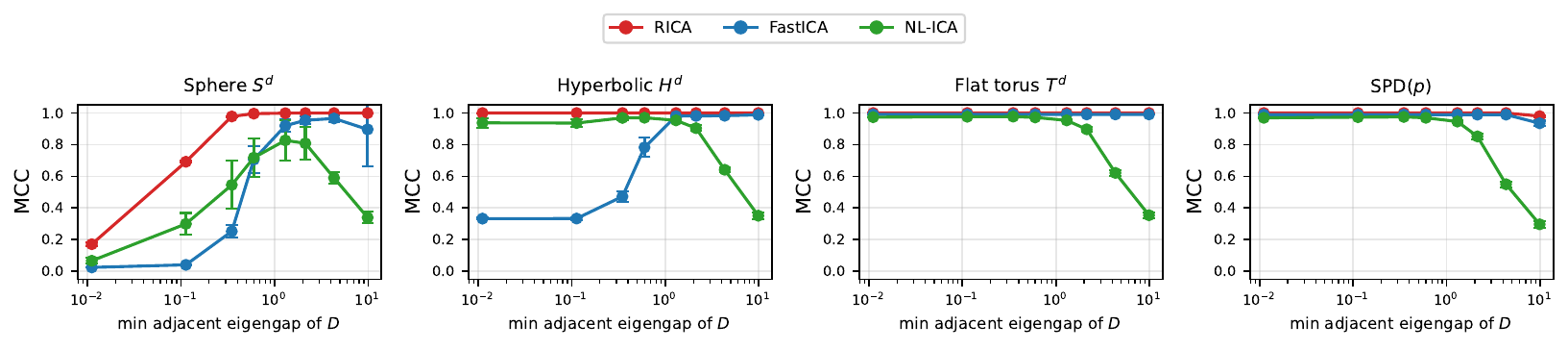}
\caption{MCC versus the minimum adjacent eigengap of the disentanglement tensor.
  The sweep uses the same target dimensions as \Cref{tab:source-recovery}.
  The gap is swept by varying $r_s \in \{0.999, 0.99, 0.97, 0.95, 0.9, 0.85, 0.75, 0.6\}$ at fixed $b=0.3$.
  RICA and FastICA improve as the eigenvalues become more separated, while NL-ICA degrades at the largest eigengaps, where the geometric source scales are most spread out.}
\label{fig:source-recovery-eigvalgap}
\end{figure}

In \Cref{fig:source-recovery-eigvalgap}, we show how the methods behave as the eigenvalues of the disentanglement tensor get closer together.
FastICA degrades quickly as the eigengap closes, and RICA remains accurate down to a small gap before its second-order identifiability fails.
NL-ICA degrades at the largest eigengaps in every panel.
Those eigengaps are produced by the smallest $r_s$, which spreads the source scales, and the maximum-likelihood training then fails on the small-scale sources.

\Cref{fig:source-recovery-b-sweep} shows a failure mode of RICA when the samples are generated outside the injectivity radius of the exponential map.
As the base scale $b$ increases, the samples move farther from the base point, where the exponential map may no longer be invertible.
Here, source recovery with RICA fails because the logarithmic map cannot invert the exponential map.

\section{Related work}
\label{sec:related}

There is no single definition of disentanglement, and the concept appears in many different contexts \cite{wang2024disentangled}.
A large portion of the literature studies generative models with structured latent variables.
This includes ICA and nonlinear ICA (\Cref{sec:ica}), GAN-based methods \cite{goodfellow2014generative,chen2016infogan}, and VAE-based methods \cite{kingma2014autoencoding,higgins_beta-vae_2017,burgess2018understanding,kim2018disentangling,chen_isolating_2018,kumar2018variational,reizinger2022embrace}.
These approaches typically define factors as statistically independent latent variables that generate the data.
Other approaches define disentanglement without starting from statistical independence of latent coordinates.
Group-theoretic approaches characterize disentangled representations by how transformations act independently on latent subspaces \cite{higgins2018towards,caselles2019symmetry}.
Causal approaches instead tie factors to mechanisms and interventions \cite{suter2019robustly,yang_causalvae_2021,scholkopf2021toward}.
RICA is closest to ICA and nonlinear ICA because it defines disentanglement through an invertible map from data to latent variables.
However, RICA finds local factors of variation from a metric and a density rather than from a global latent-variable model.

Recently, interpretability research in large language models has studied features that exhibit disentangled behavior \cite{o2024disentangling,elhage2022toy,templeton2024scaling}.
In particular, sparse autoencoders have been used to decompose neural representations into interpretable features \cite{cunningham2023sparseautoencodershighlyinterpretable,sharkey2025open}.
However, these features are not always isolated factors of variation \cite{mueller2025isolation}.
Although RICA is also motivated by interpretability, it does not offer a global interpretation of features in the way these methods do and instead recovers local factors of variation.
An interesting future research direction could be to construct global interpretations that locally agree with RICA.

Geometric tools are widely used in machine learning \cite{bronstein2021geometric,gruffaz2025riemannian,dombrowski2021diffeomorphic,williams2025geodesic}.
Some methods study learning problems on manifolds other than $\mathbb{R}^n$ \cite{gemici2016normalizing,chen2024flow}, whereas our experiments focus on Euclidean data spaces.
This focus does not restrict the applicability of RICA because normal coordinates exist on every smooth Riemannian manifold \cite{lee2018introduction}, and even in Euclidean space, the choice of metric can induce nontrivial geometric structure.
Other methods learn structure from data, distributions, or generative models \cite{gruffaz2025riemannian,bamberger2026riemannian,hartmann2022lagrangian,williams2025geodesic,arvanitidis2018latent,jones2026computing}.
RICA takes the geometry as input and can therefore be applied with the metrics learned by these methods.

The geometric methods most closely related to RICA are generalizations of PCA to manifolds.
Tangent PCA \cite{pennec2006intrinsic}, principal geodesic analysis \cite{fletcher2004principal,ppga,zhang2019mixture}, and barycentric subspace analysis \cite{pennec2018barycentric} extend PCA to manifold-valued data for dimensionality reduction.
Like RICA, these methods recover directions or subspaces that summarize local variation in the data.
However, they optimize variance or reconstruction criteria \cite{pennec2006intrinsic,fletcher2011geodesic}, whereas RICA uses an eigendecomposition of the second derivatives of the log density in normal coordinates to find local factors of variation.

\section{Conclusion}
\label{sec:conclusion}

We introduced Riemannian ICA as a local geometric extension of ICA that replaces the global invertible generative model with a Riemannian metric and a data density.
This allows us to interpret normal coordinates around a base point as adaptive latent variables and factors of variation as geodesics.
Our central theoretical contribution is the disentanglement tensor, $\mathcal{D} = \nabla^2 \log \rho - \tfrac{1}{3}\Ric$, whose eigenspaces identify pointwise disentangled directions.
This gives an algorithm for pointwise disentanglement by solving a generalized eigenvalue problem.
In a controlled source recovery setting where the metric, likelihood, and ground truth sources are known, RICA recovers the tangent space sources across several tractable manifolds and coordinate charts.
An open challenge is to scale RICA to modern pretrained encoders and determine whether the resulting local factors are useful for interpreting model representations.
If this challenge can be met, RICA may provide a route toward theoretically grounded interpretability of pretrained representations.

\begin{ack}
  The author thanks Arie Stern Gonzalez, Carlos Sotos, Daniel Sheldon, John Raisbeck, Madalina Fiterau, Miguel Fuentes, and Will Schwarzer for helpful discussions and feedback.
\end{ack}

\bibliographystyle{plainnat}
\bibliography{references}

\begin{thebibliography}{89}
\providecommand{\natexlab}[1]{#1}
\providecommand{\url}[1]{\texttt{#1}}
\expandafter\ifx\csname urlstyle\endcsname\relax
  \providecommand{\doi}[1]{doi: #1}\else
  \providecommand{\doi}{doi: \begingroup \urlstyle{rm}\Url}\fi

\bibitem[Arvanitidis et~al.(2018)Arvanitidis, Hansen, and
  Hauberg]{arvanitidis2018latent}
Georgios Arvanitidis, Lars~Kai Hansen, and S{\o}ren Hauberg.
\newblock Latent space oddity: on the curvature of deep generative models.
\newblock In \emph{International Conference on Learning Representations}, 2018.
\newblock URL \url{https://openreview.net/forum?id=SJzRZ-WCZ}.

\bibitem[Assran et~al.(2023)Assran, Duval, Misra, Bojanowski, Vincent, Rabbat,
  LeCun, and Ballas]{assran2023self}
Mahmoud Assran, Quentin Duval, Ishan Misra, Piotr Bojanowski, Pascal Vincent,
  Michael Rabbat, Yann LeCun, and Nicolas Ballas.
\newblock Self-supervised learning from images with a joint-embedding
  predictive architecture.
\newblock In \emph{Proceedings of the IEEE/CVF conference on computer vision
  and pattern recognition}, pages 15619--15629, 2023.

\bibitem[Balestriero et~al.(2025)Balestriero, Ballas, Rabbat, and
  LeCun]{balestriero2025gaussian}
Randall Balestriero, Nicolas Ballas, Mike Rabbat, and Yann LeCun.
\newblock Gaussian embeddings: How jepas secretly learn your data density.
\newblock \emph{arXiv preprint arXiv:2510.05949}, 2025.

\bibitem[Bamberger et~al.(2026)Bamberger, Gosztolai, Vandergheynst, Bronstein,
  and Jones]{bamberger2026riemannian}
Jacob Bamberger, Adam Gosztolai, Pierre Vandergheynst, Michael~M. Bronstein,
  and Iolo Jones.
\newblock Riemannian metric matching for scalable geometric modelling of
  distributions.
\newblock In \emph{ICLR 2026 Workshop on Geometry-grounded Representation
  Learning and Generative Modeling}, 2026.
\newblock URL \url{https://openreview.net/forum?id=DB24JOyswh}.

\bibitem[Bengio et~al.(2013)Bengio, Courville, and
  Vincent]{bengio2013representation}
Yoshua Bengio, Aaron Courville, and Pascal Vincent.
\newblock Representation learning: A review and new perspectives.
\newblock \emph{IEEE transactions on pattern analysis and machine
  intelligence}, 35\penalty0 (8):\penalty0 1798--1828, 8 2013.
\newblock ISSN 0162-8828.
\newblock \doi{10.1109/tpami.2013.50}.

\bibitem[Bradbury et~al.(2018)Bradbury, Frostig, Hawkins, Johnson, Katariya,
  Leary, Maclaurin, Necula, Paszke, Vander{P}las, Wanderman-{M}ilne, and
  Zhang]{jax2018github}
James Bradbury, Roy Frostig, Peter Hawkins, Matthew~James Johnson, Yash
  Katariya, Chris Leary, Dougal Maclaurin, George Necula, Adam Paszke, Jake
  Vander{P}las, Skye Wanderman-{M}ilne, and Qiao Zhang.
\newblock {JAX}: composable transformations of {P}ython+{N}um{P}y programs,
  2018.
\newblock URL \url{http://github.com/jax-ml/jax}.

\bibitem[Brady et~al.(2024)Brady, von Kügelgen, Lachapelle, Buchholz, Kipf,
  and Brendel]{Brady2024interaction}
Jack Brady, Julius von Kügelgen, Sébastien Lachapelle, Simon Buchholz, Thomas
  Kipf, and Wieland Brendel.
\newblock Interaction asymmetry: A general principle for learning composable
  abstractions.
\newblock \emph{arXiv}, 2024.
\newblock URL \url{https://arxiv.org/abs/2411.07784}.

\bibitem[Bronstein et~al.(2021)Bronstein, Bruna, Cohen, and
  Veli{\v{c}}kovi{\'c}]{bronstein2021geometric}
Michael~M Bronstein, Joan Bruna, Taco Cohen, and Petar Veli{\v{c}}kovi{\'c}.
\newblock Geometric deep learning: Grids, groups, graphs, geodesics, and
  gauges.
\newblock \emph{arXiv preprint arXiv:2104.13478}, 2021.

\bibitem[Buchholz et~al.(2022)Buchholz, Besserve, and
  Sch{\"o}lkopf]{buchholz2022function}
Simon Buchholz, Michel Besserve, and Bernhard Sch{\"o}lkopf.
\newblock Function classes for identifiable nonlinear independent component
  analysis.
\newblock \emph{Advances in Neural Information Processing Systems},
  35:\penalty0 16946--16961, 8 2022.
\newblock \doi{10.48550/arxiv.2208.06406}.

\bibitem[Buchholz et~al.(2023)Buchholz, Rajendran, Rosenfeld, Aragam,
  Sch{\"o}lkopf, and Ravikumar]{buchholz2023learning}
Simon Buchholz, Goutham Rajendran, Elan Rosenfeld, Bryon Aragam, Bernhard
  Sch{\"o}lkopf, and Pradeep Ravikumar.
\newblock Learning linear causal representations from interventions under
  general nonlinear mixing.
\newblock \emph{Advances in Neural Information Processing Systems},
  36:\penalty0 45419--45462, 2023.

\bibitem[Burgess et~al.(2017)Burgess, Higgins, Pal, Matthey, Watters,
  Desjardins, and Lerchner]{burgess2018understanding}
Christopher~P. Burgess, Irina Higgins, Arka Pal, Lo{\"i}c Matthey, Nick
  Watters, Guillaume Desjardins, and Alexander Lerchner.
\newblock Understanding disentangling in {$\beta$}-{VAE}.
\newblock In \emph{NIPS 2017 Workshop on Learning Disentangled
  Representations}, 2017.

\bibitem[Caron et~al.(2021)Caron, Touvron, Misra, J\'egou, Mairal, Bojanowski,
  and Joulin]{caron2021emerging}
Mathilde Caron, Hugo Touvron, Ishan Misra, Herv\'e J\'egou, Julien Mairal,
  Piotr Bojanowski, and Armand Joulin.
\newblock Emerging properties in self-supervised vision transformers.
\newblock In \emph{Proceedings of the International Conference on Computer
  Vision (ICCV)}, 2021.

\bibitem[Caselles-Dupr{\'e} et~al.(2019)Caselles-Dupr{\'e}, Garcia~Ortiz, and
  Filliat]{caselles2019symmetry}
Hugo Caselles-Dupr{\'e}, Michael Garcia~Ortiz, and David Filliat.
\newblock Symmetry-based disentangled representation learning requires
  interaction with environments.
\newblock \emph{Advances in Neural Information Processing Systems}, 32, 2019.

\bibitem[Chen et~al.(2018)Chen, Li, Grosse, and Duvenaud]{chen_isolating_2018}
Ricky T.~Q. Chen, Xuechen Li, Roger~B Grosse, and David~K Duvenaud.
\newblock Isolating {Sources} of {Disentanglement} in {Variational}
  {Autoencoders}.
\newblock In S.~Bengio, H.~Wallach, H.~Larochelle, K.~Grauman, N.~Cesa-Bianchi,
  and R.~Garnett, editors, \emph{Advances in {Neural} {Information}
  {Processing} {Systems}}, volume~31. Curran Associates, Inc., 2 2018.
\newblock URL
  \url{https://proceedings.neurips.cc/paper/2018/file/1ee3dfcd8a0645a25a35977997223d22-Paper.pdf}.

\bibitem[Chen and Lipman(2024)]{chen2024flow}
Ricky~TQ Chen and Yaron Lipman.
\newblock Flow matching on general geometries.
\newblock In \emph{International Conference on Learning Representations},
  volume 2024, pages 47922--47945, 2024.

\bibitem[Chen et~al.(2016)Chen, Duan, Houthooft, Schulman, Sutskever, and
  Abbeel]{chen2016infogan}
Xi~Chen, Yan Duan, Rein Houthooft, John Schulman, Ilya Sutskever, and Pieter
  Abbeel.
\newblock Infogan: Interpretable representation learning by information
  maximizing generative adversarial nets.
\newblock \emph{Advances in neural information processing systems}, 29, 2016.

\bibitem[Comon(1994)]{comon1994independent}
Pierre Comon.
\newblock Independent component analysis, a new concept?
\newblock \emph{Signal Processing}, 36\penalty0 (3):\penalty0 287--314, 1994.
\newblock \doi{10.1016/0165-1684(94)90029-9}.

\bibitem[Cornish et~al.(2020)Cornish, Caterini, Deligiannidis, and
  Doucet]{cornish2020relaxing}
Rob Cornish, Anthony Caterini, George Deligiannidis, and Arnaud Doucet.
\newblock Relaxing bijectivity constraints with continuously indexed
  normalising flows.
\newblock In \emph{International conference on machine learning}, pages
  2133--2143. PMLR, 2020.

\bibitem[Cunningham(2026)]{cunningham2026localcoordinates}
Edmond Cunningham.
\newblock {local\_coordinates}: A {JAX}-based framework for differential
  geometry computations on {R}iemannian manifolds, 2026.
\newblock URL \url{https://github.com/EddieCunningham/local_coordinates}.

\bibitem[DeepMind et~al.(2020)DeepMind, Babuschkin, Baumli, Bell, Bhupatiraju,
  Bruce, Buchlovsky, Budden, Cai, Clark, Danihelka, Dedieu, Fantacci, Godwin,
  Jones, Hemsley, Hennigan, Hessel, Hou, Kapturowski, Keck, Kemaev, King,
  Kunesch, Martens, Merzic, Mikulik, Norman, Papamakarios, Quan, Ring, Ruiz,
  Sanchez, Sartran, Schneider, Sezener, Spencer, Srinivasan, Stanojevi{\'c},
  Stokowiec, Wang, Zhou, and Viola]{deepmind2020jax}
DeepMind, Igor Babuschkin, Kate Baumli, Alison Bell, Surya Bhupatiraju, Jake
  Bruce, Peter Buchlovsky, David Budden, Trevor Cai, Aidan Clark, Ivo
  Danihelka, Antoine Dedieu, Claudio Fantacci, Jonathan Godwin, Chris Jones,
  Ross Hemsley, Tom Hennigan, Matteo Hessel, Shaobo Hou, Steven Kapturowski,
  Thomas Keck, Iurii Kemaev, Michael King, Markus Kunesch, Lena Martens, Hamza
  Merzic, Vladimir Mikulik, Tamara Norman, George Papamakarios, John Quan,
  Roman Ring, Francisco Ruiz, Alvaro Sanchez, Laurent Sartran, Rosalia
  Schneider, Eren Sezener, Stephen Spencer, Srivatsan Srinivasan, Milo{\v{s}}
  Stanojevi{\'c}, Wojciech Stokowiec, Luyu Wang, Guangyao Zhou, and Fabio
  Viola.
\newblock The {D}eep{M}ind {JAX} {E}cosystem, 2020.
\newblock URL \url{http://github.com/google-deepmind/optax}.

\bibitem[Deng et~al.(2009)Deng, Dong, Socher, Li, Li, and Fei-Fei]{imagenet}
Jia Deng, Wei Dong, Richard Socher, Li-Jia Li, Kai Li, and Li~Fei-Fei.
\newblock Imagenet: A large-scale hierarchical image database.
\newblock In \emph{2009 IEEE Conference on Computer Vision and Pattern
  Recognition}, pages 248--255, 2009.
\newblock \doi{10.1109/CVPR.2009.5206848}.

\bibitem[Dinh et~al.(2017)Dinh, Sohl-Dickstein, and Bengio]{dinh2017density}
Laurent Dinh, Jascha Sohl-Dickstein, and Samy Bengio.
\newblock Density estimation using real {NVP}.
\newblock In \emph{International Conference on Learning Representations}, 2017.
\newblock URL \url{https://openreview.net/forum?id=HkpbnH9lx}.

\bibitem[Dombrowski et~al.(2021)Dombrowski, Gerken, and
  Kessel]{dombrowski2021diffeomorphic}
Ann-Kathrin Dombrowski, Jan~E Gerken, and Pan Kessel.
\newblock Diffeomorphic explanations with normalizing flows.
\newblock In \emph{ICML Workshop on Invertible Neural Networks, Normalizing
  Flows, and Explicit Likelihood Models}, 6 2021.
\newblock URL \url{https://openreview.net/forum?id=ZBR9EpEl6G4}.

\bibitem[Elhage et~al.(2022)Elhage, Hume, Olsson, Schiefer, Henighan, Kravec,
  Hatfield-Dodds, Lasenby, Drain, Chen, Grosse, McCandlish, Kaplan, Amodei,
  Wattenberg, and Olah]{elhage2022toy}
Nelson Elhage, Tristan Hume, Catherine Olsson, Nicholas Schiefer, Tom Henighan,
  Shauna Kravec, Zac Hatfield-Dodds, Robert Lasenby, Dawn Drain, Carol Chen,
  Roger Grosse, Sam McCandlish, Jared Kaplan, Dario Amodei, Martin Wattenberg,
  and Christopher Olah.
\newblock Toy models of superposition.
\newblock \url{https://transformer-circuits.pub/2022/toy_model/index.html},
  2022.
\newblock URL \url{https://transformer-circuits.pub/2022/toy_model/index.html}.
\newblock Transformer Circuits Thread.

\bibitem[Fletcher et~al.(2004)Fletcher, Lu, Pizer, and
  Joshi]{fletcher2004principal}
P~Thomas Fletcher, Conglin Lu, Stephen~M Pizer, and Sarang Joshi.
\newblock Principal geodesic analysis for the study of nonlinear statistics of
  shape.
\newblock \emph{IEEE transactions on medical imaging}, 23\penalty0
  (8):\penalty0 995--1005, 2004.

\bibitem[Fletcher(2011)]{fletcher2011geodesic}
Thomas Fletcher.
\newblock Geodesic regression on riemannian manifolds.
\newblock In \emph{Proceedings of the Third International Workshop on
  Mathematical Foundations of Computational Anatomy-Geometrical and Statistical
  Methods for Modelling Biological Shape Variability}, pages 75--86, 2011.

\bibitem[Gemici et~al.(2016)Gemici, Rezende, and
  Mohamed]{gemici2016normalizing}
Mevlana~C Gemici, Danilo Rezende, and Shakir Mohamed.
\newblock Normalizing flows on riemannian manifolds.
\newblock \emph{arXiv preprint arXiv:1611.02304}, 2016.

\bibitem[Golub and Van~Loan(2013)]{golub2013matrix}
Gene~H Golub and Charles~F Van~Loan.
\newblock \emph{Matrix computations}.
\newblock JHU press, 2013.

\bibitem[Goodfellow et~al.(2014)Goodfellow, Pouget-Abadie, Mirza, Xu,
  Warde-Farley, Ozair, Courville, and Bengio]{goodfellow2014generative}
Ian~J Goodfellow, Jean Pouget-Abadie, Mehdi Mirza, Bing Xu, David Warde-Farley,
  Sherjil Ozair, Aaron Courville, and Yoshua Bengio.
\newblock Generative adversarial nets.
\newblock \emph{Advances in neural information processing systems}, 27, 2014.

\bibitem[Gresele et~al.(2021)Gresele, K{\"u}gelgen, Stimper, Sch{\"o}lkopf, and
  Besserve]{gresele2021independent}
Luigi Gresele, Julius~Von K{\"u}gelgen, Vincent Stimper, Bernhard
  Sch{\"o}lkopf, and Michel Besserve.
\newblock Independent mechanism analysis, a new concept?
\newblock In A.~Beygelzimer, Y.~Dauphin, P.~Liang, and J.~Wortman Vaughan,
  editors, \emph{Advances in Neural Information Processing Systems}, pages
  28233--28248, 2021.
\newblock URL \url{https://openreview.net/forum?id=Rnn8zoAkrwr}.

\bibitem[Gruffaz and Sassen(2025)]{gruffaz2025riemannian}
Samuel Gruffaz and Josua Sassen.
\newblock Riemannian metric learning: Closer to you than you imagine.
\newblock \emph{arXiv preprint arXiv:2503.05321}, 2025.

\bibitem[Harris et~al.(2020)Harris, Millman, van~der Walt, Gommers, Virtanen,
  Cournapeau, Wieser, Taylor, Berg, Smith, Kern, Picus, Hoyer, van Kerkwijk,
  Brett, Haldane, del R{\'i}o, Wiebe, Peterson, G{\'e}rard-Marchant, Sheppard,
  Reddy, Weckesser, Abbasi, Gohlke, and Oliphant]{harris2020array}
Charles~R. Harris, K.~Jarrod Millman, St{\'e}fan~J. van~der Walt, Ralf Gommers,
  Pauli Virtanen, David Cournapeau, Eric Wieser, Julian Taylor, Sebastian Berg,
  Nathaniel~J. Smith, Robert Kern, Matti Picus, Stephan Hoyer, Marten~H. van
  Kerkwijk, Matthew Brett, Allan Haldane, Jaime~Fern{\'a}ndez del R{\'i}o, Mark
  Wiebe, Pearu Peterson, Pierre G{\'e}rard-Marchant, Kevin Sheppard, Tyler
  Reddy, Warren Weckesser, Hameer Abbasi, Christoph Gohlke, and Travis~E.
  Oliphant.
\newblock Array programming with {NumPy}.
\newblock \emph{Nature}, 585\penalty0 (7825):\penalty0 357--362, 2020.
\newblock \doi{10.1038/s41586-020-2649-2}.

\bibitem[Hartmann et~al.(2022)Hartmann, Girolami, and
  Klami]{hartmann2022lagrangian}
Marcelo Hartmann, Mark Girolami, and Arto Klami.
\newblock Lagrangian manifold monte carlo on monge patches.
\newblock In \emph{International Conference on Artificial Intelligence and
  Statistics}, pages 4764--4781. PMLR, 2022.

\bibitem[Higgins et~al.(2017)Higgins, Matthey, Pal, Burgess, Glorot, Botvinick,
  Mohamed, and Lerchner]{higgins_beta-vae_2017}
Irina Higgins, Loic Matthey, Arka Pal, Christopher Burgess, Xavier Glorot,
  Matthew Botvinick, Shakir Mohamed, and Alexander Lerchner.
\newblock beta-vae: {Learning} basic visual concepts with a constrained
  variational framework.
\newblock In \emph{International Conference on Learning Representations}, 4
  2017.

\bibitem[Higgins et~al.(2018)Higgins, Amos, Pfau, Racaniere, Matthey, Rezende,
  and Lerchner]{higgins2018towards}
Irina Higgins, David Amos, David Pfau, Sebastien Racaniere, Loic Matthey,
  Danilo Rezende, and Alexander Lerchner.
\newblock Towards a definition of disentangled representations.
\newblock \emph{arXiv preprint arXiv:1812.02230}, abs/1812.02230, 2018.

\bibitem[Horan et~al.(2021)Horan, Richardson, and Weiss]{horan2021unsupervised}
Daniella Horan, Eitan Richardson, and Yair Weiss.
\newblock When is unsupervised disentanglement possible?
\newblock \emph{Advances in Neural Information Processing Systems},
  34:\penalty0 5150--5161, 2021.

\bibitem[Huben et~al.(2024)Huben, Cunningham, Smith, Ewart, and
  Sharkey]{cunningham2023sparseautoencodershighlyinterpretable}
Robert Huben, Hoagy Cunningham, Logan~Riggs Smith, Aidan Ewart, and Lee
  Sharkey.
\newblock Sparse autoencoders find highly interpretable features in language
  models.
\newblock In \emph{The Twelfth International Conference on Learning
  Representations}, 2024.
\newblock URL \url{https://openreview.net/forum?id=F76bwRSLeK}.

\bibitem[Hunter(2007)]{Hunter:2007}
J.~D. Hunter.
\newblock Matplotlib: A 2d graphics environment.
\newblock \emph{Computing in Science \& Engineering}, 9\penalty0 (3):\penalty0
  90--95, 2007.
\newblock \doi{10.1109/MCSE.2007.55}.

\bibitem[Hyv{\"a}rinen and Oja(2000)]{hyvarinen2000ica}
Aapo Hyv{\"a}rinen and Erkki Oja.
\newblock Independent component analysis: Algorithms and applications.
\newblock \emph{Neural Networks}, 13\penalty0 (4--5):\penalty0 411--430, 6
  2000.
\newblock ISSN 0893-6080.
\newblock \doi{10.1016/S0893-6080(00)00026-5}.

\bibitem[Hyv\"{a}rinen and Pajunen(1999)]{nonlinear_ica}
Aapo Hyv\"{a}rinen and Petteri Pajunen.
\newblock Nonlinear independent component analysis: Existence and uniqueness
  results.
\newblock \emph{Neural Netw.}, 12\penalty0 (3):\penalty0 429–439, apr 1999.
\newblock ISSN 0893-6080.
\newblock \doi{10.1016/S0893-6080(98)00140-3}.
\newblock URL \url{https://doi.org/10.1016/S0893-6080(98)00140-3}.

\bibitem[Hyvarinen et~al.(2019)Hyvarinen, Sasaki, and
  Turner]{hyvarinen2019nonlinear}
Aapo Hyvarinen, Hiroaki Sasaki, and Richard Turner.
\newblock Nonlinear ica using auxiliary variables and generalized contrastive
  learning.
\newblock In \emph{The 22nd international conference on artificial intelligence
  and statistics}, pages 859--868. PMLR, 2019.

\bibitem[Hyv{\"a}rinen et~al.(2023)Hyv{\"a}rinen, Khemakhem, and
  Morioka]{hyvarinen2023nonlinear}
Aapo Hyv{\"a}rinen, Ilyes Khemakhem, and Hiroshi Morioka.
\newblock Nonlinear independent component analysis for principled
  disentanglement in unsupervised deep learning.
\newblock \emph{Patterns}, 4\penalty0 (10), 2023.

\bibitem[Hyvärinen and Oja(1997)]{fastica}
Aapo Hyvärinen and Erkki Oja.
\newblock A fast fixed-point algorithm for independent component analysis.
\newblock \emph{Neural Computation}, 9\penalty0 (7):\penalty0 1483--1492, 1997.
\newblock \doi{10.1162/neco.1997.9.7.1483}.

\bibitem[Jones and Lanners(2026)]{jones2026computing}
Iolo Jones and David Lanners.
\newblock Computing diffusion geometry.
\newblock \emph{arXiv preprint arXiv:2602.06006}, 2026.

\bibitem[Karras et~al.(2019)Karras, Laine, and Aila]{karras2019style}
Tero Karras, Samuli Laine, and Timo Aila.
\newblock A style-based generator architecture for generative adversarial
  networks.
\newblock In \emph{Proceedings of the IEEE/CVF Conference on Computer Vision
  and Pattern Recognition}, pages 4401--4410, 6 2019.
\newblock \doi{10.1109/cvpr.2019.00453}.

\bibitem[Khemakhem et~al.(2020)Khemakhem, Kingma, Monti, and
  Hyvarinen]{khemakhem2020variational}
Ilyes Khemakhem, Diederik Kingma, Ricardo Monti, and Aapo Hyvarinen.
\newblock Variational autoencoders and nonlinear ica: A unifying framework.
\newblock In \emph{International conference on artificial intelligence and
  statistics}, pages 2207--2217. PMLR, 2020.

\bibitem[Kidger and Garcia(2021)]{kidger2021equinox}
Patrick Kidger and Cristian Garcia.
\newblock {E}quinox: neural networks in {JAX} via callable {P}y{T}rees and
  filtered transformations.
\newblock \emph{Differentiable Programming workshop at Neural Information
  Processing Systems 2021}, 2021.

\bibitem[Kim and Mnih(2018)]{kim2018disentangling}
Hyunjik Kim and Andriy Mnih.
\newblock Disentangling by factorising.
\newblock In \emph{International conference on machine learning}, pages
  2649--2658. PMLR, 2018.

\bibitem[Kingma and Welling(2014)]{kingma2014autoencoding}
Diederik~P. Kingma and Max Welling.
\newblock Auto-encoding variational bayes.
\newblock In Yoshua Bengio and Yann LeCun, editors, \emph{2nd International
  Conference on Learning Representations, {ICLR} 2014, Banff, AB, Canada, April
  14-16, 2014, Conference Track Proceedings}, 2014.
\newblock URL \url{http://arxiv.org/abs/1312.6114}.

\bibitem[Kraskov et~al.(2004)Kraskov, St{\"o}gbauer, and
  Grassberger]{kraskov2004estimating}
Alexander Kraskov, Harald St{\"o}gbauer, and Peter Grassberger.
\newblock Estimating mutual information.
\newblock \emph{Physical Review E—Statistical, Nonlinear, and Soft Matter
  Physics}, 69\penalty0 (6):\penalty0 066138, 2004.

\bibitem[Krishnan(2021)]{krishnan2021diagonalizing}
Anusha~M Krishnan.
\newblock Diagonalizing the ricci tensor.
\newblock \emph{The Journal of Geometric Analysis}, 31\penalty0 (6):\penalty0
  5638--5658, 2021.

\bibitem[Kumar et~al.(2018)Kumar, Sattigeri, and
  Balakrishnan]{kumar2018variational}
Abhishek Kumar, Prasanna Sattigeri, and Avinash Balakrishnan.
\newblock Variational inference of disentangled latent concepts from unlabeled
  observations.
\newblock In \emph{International Conference on Learning Representations}, 2018.

\bibitem[Lachapelle et~al.(2022)Lachapelle, Rodriguez, Sharma, Everett,
  Le~Priol, Lacoste, and Lacoste-Julien]{lachapelle2022disentanglement}
S{\'e}bastien Lachapelle, Pau Rodriguez, Yash Sharma, Katie~E Everett, R{\'e}mi
  Le~Priol, Alexandre Lacoste, and Simon Lacoste-Julien.
\newblock Disentanglement via mechanism sparsity regularization: A new
  principle for nonlinear ica.
\newblock In \emph{Conference on Causal Learning and Reasoning}, pages
  428--484. PMLR, 2022.

\bibitem[Lee(2000)]{Lee00}
John~M. Lee.
\newblock \emph{Introduction to Smooth Manifolds}.
\newblock Springer, 2000.

\bibitem[Lee(2018)]{lee2018introduction}
John~M. Lee.
\newblock \emph{Introduction to Riemannian Manifolds}, volume 176 of
  \emph{Graduate Texts in Mathematics}.
\newblock Springer, 2nd edition, 2018.
\newblock ISBN 978-3-319-91754-2.
\newblock \doi{10.1007/978-3-319-91755-9}.

\bibitem[Locatello et~al.(2020)Locatello, Bauer, Lucic, R{\"a}tsch, Gelly,
  Sch{\"o}lkopf, and Bachem]{locatello2020sober}
Francesco Locatello, Stefan Bauer, Mario Lucic, Gunnar R{\"a}tsch, Sylvain
  Gelly, Bernhard Sch{\"o}lkopf, and Olivier Bachem.
\newblock A sober look at the unsupervised learning of disentangled
  representations and their evaluation.
\newblock \emph{Journal of Machine Learning Research}, 21\penalty0
  (209):\penalty0 1--62, 2020.

\bibitem[Matthes et~al.(2026)Matthes, Han, and Shen]{matthes2026mechanistic}
Stefan Matthes, Zhiwei Han, and Hao Shen.
\newblock Mechanistic independence: A principle for identifiable disentangled
  representations.
\newblock In \emph{The Fourteenth International Conference on Learning
  Representations}, 2026.
\newblock URL \url{https://openreview.net/forum?id=0VVdai71xb}.

\bibitem[Mikolov et~al.(2013)Mikolov, Yih, and
  Zweig]{mikolov-etal-2013-linguistic}
Tomas Mikolov, Wen-tau Yih, and Geoffrey Zweig.
\newblock Linguistic regularities in continuous space word representations.
\newblock In \emph{Proceedings of the 2013 Conference of the North {A}merican
  Chapter of the Association for Computational Linguistics: Human Language
  Technologies}, pages 746--751, Atlanta, Georgia, June 2013. Association for
  Computational Linguistics.
\newblock URL \url{https://aclanthology.org/N13-1090}.

\bibitem[Moser(1965)]{moser1965volume}
J{\"u}rgen Moser.
\newblock On the volume elements on a manifold.
\newblock \emph{Transactions of the American Mathematical Society},
  120\penalty0 (2):\penalty0 286--294, 1965.

\bibitem[Mueller et~al.(2025)Mueller, Lee, Joshi, Lubana, Sridhar, and
  Reizinger]{mueller2025isolation}
Aaron Mueller, Andrew Lee, Shruti Joshi, Ekdeep~Singh Lubana, Dhanya Sridhar,
  and Patrik Reizinger.
\newblock From isolation to entanglement: When do interpretability methods
  identify and disentangle known concepts?
\newblock \emph{arXiv preprint arXiv:2512.15134}, 2025.

\bibitem[O'Neill et~al.(2024)O'Neill, Ye, Iyer, and Wu]{o2024disentangling}
Charles O'Neill, Christine Ye, Kartheik Iyer, and John~F Wu.
\newblock Disentangling dense embeddings with sparse autoencoders.
\newblock \emph{arXiv preprint arXiv:2408.00657}, 2024.

\bibitem[Oquab et~al.(2024)Oquab, Darcet, Moutakanni, Vo, Szafraniec, Khalidov,
  Fernandez, HAZIZA, Massa, El-Nouby, Assran, Ballas, Galuba, Howes, Huang, Li,
  Misra, Rabbat, Sharma, Synnaeve, Xu, Jegou, Mairal, Labatut, Joulin, and
  Bojanowski]{oquab2024dinov}
Maxime Oquab, Timoth{\'e}e Darcet, Th{\'e}o Moutakanni, Huy~V. Vo, Marc
  Szafraniec, Vasil Khalidov, Pierre Fernandez, Daniel HAZIZA, Francisco Massa,
  Alaaeldin El-Nouby, Mido Assran, Nicolas Ballas, Wojciech Galuba, Russell
  Howes, Po-Yao Huang, Shang-Wen Li, Ishan Misra, Michael Rabbat, Vasu Sharma,
  Gabriel Synnaeve, Hu~Xu, Herve Jegou, Julien Mairal, Patrick Labatut, Armand
  Joulin, and Piotr Bojanowski.
\newblock {DINO}v2: Learning robust visual features without supervision.
\newblock \emph{Transactions on Machine Learning Research}, 2024.
\newblock ISSN 2835-8856.
\newblock URL \url{https://openreview.net/forum?id=a68SUt6zFt}.
\newblock Featured Certification.

\bibitem[Papamakarios et~al.(2021)Papamakarios, Nalisnick, Rezende, Mohamed,
  and Lakshminarayanan]{papamakarios2021normalizing}
George Papamakarios, Eric Nalisnick, Danilo~Jimenez Rezende, Shakir Mohamed,
  and Balaji Lakshminarayanan.
\newblock Normalizing flows for probabilistic modeling and inference.
\newblock \emph{Journal of Machine Learning Research}, 22\penalty0
  (57):\penalty0 1--64, 2021.

\bibitem[Park et~al.(2023)Park, Choe, and
  Veitch]{park2024linearrepresentationhypothesisgeometry}
Kiho Park, Yo~Joong Choe, and Victor Veitch.
\newblock The linear representation hypothesis and the geometry of large
  language models.
\newblock In \emph{Causal Representation Learning Workshop at NeurIPS 2023},
  2023.
\newblock URL \url{https://openreview.net/forum?id=T0PoOJg8cK}.

\bibitem[Pedregosa et~al.(2011)Pedregosa, Varoquaux, Gramfort, Michel, Thirion,
  Grisel, Blondel, Prettenhofer, Weiss, Dubourg, Vanderplas, Passos,
  Cournapeau, Brucher, Perrot, and Duchesnay]{scikit-learn}
F.~Pedregosa, G.~Varoquaux, A.~Gramfort, V.~Michel, B.~Thirion, O.~Grisel,
  M.~Blondel, P.~Prettenhofer, R.~Weiss, V.~Dubourg, J.~Vanderplas, A.~Passos,
  D.~Cournapeau, M.~Brucher, M.~Perrot, and E.~Duchesnay.
\newblock Scikit-learn: Machine learning in {P}ython.
\newblock \emph{Journal of Machine Learning Research}, 12:\penalty0 2825--2830,
  2011.

\bibitem[Pennec(2006)]{pennec2006intrinsic}
Xavier Pennec.
\newblock Intrinsic statistics on {Riemannian} manifolds: Basic tools for
  geometric measurements.
\newblock \emph{Journal of Mathematical Imaging and Vision}, 25:\penalty0
  127--154, 2006.
\newblock \doi{10.1007/s10851-006-6228-4}.

\bibitem[Pennec(2018)]{pennec2018barycentric}
Xavier Pennec.
\newblock {Barycentric subspace analysis on manifolds}.
\newblock \emph{The Annals of Statistics}, 46\penalty0 (6A):\penalty0 2711 --
  2746, 2018.
\newblock \doi{10.1214/17-AOS1636}.
\newblock URL \url{https://doi.org/10.1214/17-AOS1636}.

\bibitem[Radford et~al.(2021)Radford, Kim, Hallacy, Ramesh, Goh, Agarwal,
  Sastry, Askell, Mishkin, Clark, et~al.]{radford2021learning}
Alec Radford, Jong~Wook Kim, Chris Hallacy, Aditya Ramesh, Gabriel Goh,
  Sandhini Agarwal, Girish Sastry, Amanda Askell, Pamela Mishkin, Jack Clark,
  et~al.
\newblock Learning transferable visual models from natural language
  supervision.
\newblock In \emph{International conference on machine learning}, pages
  8748--8763. PmLR, 2021.

\bibitem[Rajendran et~al.(2024)Rajendran, Buchholz, Aragam, Sch{\"o}lkopf, and
  Ravikumar]{rajendran2024causal}
Goutham Rajendran, Simon Buchholz, Bryon Aragam, Bernhard Sch{\"o}lkopf, and
  Pradeep Ravikumar.
\newblock From causal to concept-based representation learning.
\newblock \emph{Advances in Neural Information Processing Systems},
  37:\penalty0 101250--101296, 2024.

\bibitem[Reizinger et~al.(2022)Reizinger, Gresele, Brady, Von~K{\"u}gelgen,
  Zietlow, Sch{\"o}lkopf, Martius, Brendel, and Besserve]{reizinger2022embrace}
Patrik Reizinger, Luigi Gresele, Jack Brady, Julius Von~K{\"u}gelgen, Dominik
  Zietlow, Bernhard Sch{\"o}lkopf, Georg Martius, Wieland Brendel, and Michel
  Besserve.
\newblock Embrace the gap: Vaes perform independent mechanism analysis.
\newblock \emph{Advances in Neural Information Processing Systems},
  35:\penalty0 12040--12057, 6 2022.
\newblock \doi{10.48550/arxiv.2206.02416}.

\bibitem[Rezende and Mohamed(2015)]{pmlr-v37-rezende15}
Danilo Rezende and Shakir Mohamed.
\newblock Variational inference with normalizing flows.
\newblock In Francis Bach and David Blei, editors, \emph{Proceedings of the
  32nd International Conference on Machine Learning}, volume~37 of
  \emph{Proceedings of Machine Learning Research}, pages 1530--1538, Lille,
  France, 07--09 Jul 2015. PMLR.
\newblock URL \url{https://proceedings.mlr.press/v37/rezende15.html}.

\bibitem[Rusak et~al.(2025)Rusak, Reizinger, Juhos, Bringmann, Zimmermann, and
  Brendel]{rusak2024contrastive}
Evgenia Rusak, Patrick Reizinger, Attila Juhos, Oliver Bringmann, Roland~S.
  Zimmermann, and Wieland Brendel.
\newblock Infonce: Identifying the gap between theory and practice.
\newblock In \emph{Proceedings of the 28th International Conference on
  Artificial Intelligence and Statistics}, volume 258 of \emph{Proceedings of
  Machine Learning Research}, Mai Khao, Thailand, 2025. PMLR.

\bibitem[Sch{\"o}lkopf et~al.(2021)Sch{\"o}lkopf, Locatello, Bauer, Ke,
  Kalchbrenner, Goyal, and Bengio]{scholkopf2021toward}
Bernhard Sch{\"o}lkopf, Francesco Locatello, Stefan Bauer, Nan~Rosemary Ke, Nal
  Kalchbrenner, Anirudh Goyal, and Yoshua Bengio.
\newblock Toward causal representation learning.
\newblock \emph{Proceedings of the IEEE}, 109\penalty0 (5):\penalty0 612--634,
  5 2021.
\newblock ISSN 0018-9219.
\newblock \doi{10.1109/jproc.2021.3058954}.

\bibitem[Sharkey et~al.(2025)Sharkey, Chughtai, Batson, Lindsey, Wu, Bushnaq,
  Goldowsky-Dill, Heimersheim, Ortega, Bloom, et~al.]{sharkey2025open}
Lee Sharkey, Bilal Chughtai, Joshua Batson, Jack Lindsey, Jeff Wu, Lucius
  Bushnaq, Nicholas Goldowsky-Dill, Stefan Heimersheim, Alejandro Ortega,
  Joseph Bloom, et~al.
\newblock Open problems in mechanistic interpretability.
\newblock \emph{arXiv preprint arXiv:2501.16496}, 2025.

\bibitem[Sim{\'e}oni et~al.(2025)Sim{\'e}oni, Vo, Seitzer, Baldassarre, Oquab,
  Jose, Khalidov, Szafraniec, Yi, Ramamonjisoa, Massa, Haziza, Wehrstedt, Wang,
  Darcet, Moutakanni, Sentana, Roberts, Vedaldi, Tolan, Brandt, Couprie,
  Mairal, J{\'e}gou, Labatut, and Bojanowski]{simeoni2025dinov3}
Oriane Sim{\'e}oni, Huy~V. Vo, Maximilian Seitzer, Federico Baldassarre, Maxime
  Oquab, Cijo Jose, Vasil Khalidov, Marc Szafraniec, Seungeun Yi, Micha{\"e}l
  Ramamonjisoa, Francisco Massa, Daniel Haziza, Luca Wehrstedt, Jianyuan Wang,
  Timoth{\'e}e Darcet, Th{\'e}o Moutakanni, Leonel Sentana, Claire Roberts,
  Andrea Vedaldi, Jamie Tolan, John Brandt, Camille Couprie, Julien Mairal,
  Herv{\'e} J{\'e}gou, Patrick Labatut, and Piotr Bojanowski.
\newblock {DINOv3}, 2025.
\newblock URL \url{https://arxiv.org/abs/2508.10104}.

\bibitem[Suter et~al.(2019)Suter, Miladinovic, Sch{\"o}lkopf, and
  Bauer]{suter2019robustly}
Raphael Suter, Djordje Miladinovic, Bernhard Sch{\"o}lkopf, and Stefan Bauer.
\newblock Robustly disentangled causal mechanisms: Validating deep
  representations for interventional robustness.
\newblock In \emph{International Conference on Machine Learning}, pages
  6056--6065. PMLR, 2019.

\bibitem[Templeton et~al.(2024)Templeton, Conerly, Marcus, Lindsey, Bricken,
  Chen, Pearce, Citro, Ameisen, Jones, Cunningham, Turner, McDougall,
  MacDiarmid, Freeman, Sumers, Rees, Batson, Jermyn, Carter, Olah, and
  Henighan]{templeton2024scaling}
Adly Templeton, Tom Conerly, Jonathan Marcus, Jack Lindsey, Trenton Bricken,
  Brian Chen, Adam Pearce, Craig Citro, Emmanuel Ameisen, Andy Jones, Hoagy
  Cunningham, Nicholas~L Turner, Callum McDougall, Monte MacDiarmid, C.~Daniel
  Freeman, Theodore~R. Sumers, Edward Rees, Joshua Batson, Adam Jermyn, Shan
  Carter, Chris Olah, and Tom Henighan.
\newblock Scaling monosemanticity: Extracting interpretable features from
  claude 3 sonnet.
\newblock
  \url{https://transformer-circuits.pub/2024/scaling-monosemanticity/index.html},
  2024.
\newblock URL
  \url{https://transformer-circuits.pub/2024/scaling-monosemanticity/index.html}.
\newblock Transformer Circuits Thread.

\bibitem[Tod(1992)]{tod1992choosing}
KP~Tod.
\newblock On choosing coordinates to diagonalize the metric.
\newblock \emph{Classical and Quantum Gravity}, 9\penalty0 (7):\penalty0
  1693--1705, 1992.

\bibitem[Virtanen et~al.(2020)Virtanen, Gommers, Oliphant, Haberland, Reddy,
  Cournapeau, Burovski, Peterson, Weckesser, Bright, van~der Walt, Brett,
  Wilson, Millman, Mayorov, Nelson, Jones, Kern, Larson, Carey, Polat, Feng,
  Moore, VanderPlas, Laxalde, Perktold, Cimrman, Henriksen, Quintero, Harris,
  Archibald, Ribeiro, Pedregosa, van Mulbregt, and
  Contributors]{virtanen2020scipy}
Pauli Virtanen, Ralf Gommers, Travis~E. Oliphant, Matt Haberland, Tyler Reddy,
  David Cournapeau, Evgeni Burovski, Pearu Peterson, Warren Weckesser, Jonathan
  Bright, St{\'e}fan~J. van~der Walt, Matthew Brett, Joshua Wilson, K.~Jarrod
  Millman, Nikolay Mayorov, Andrew R.~J. Nelson, Eric Jones, Robert Kern, Eric
  Larson, C.~J. Carey, {\.I}lhan Polat, Yu~Feng, Eric~W. Moore, Jake
  VanderPlas, Denis Laxalde, Josef Perktold, Robert Cimrman, Ian Henriksen,
  E.~A. Quintero, Charles~R. Harris, Anne~M. Archibald, Ant{\^o}nio~H. Ribeiro,
  Fabian Pedregosa, Paul van Mulbregt, and SciPy~1.0 Contributors.
\newblock {SciPy} 1.0: Fundamental algorithms for scientific computing in
  {Python}.
\newblock \emph{Nature Methods}, 17:\penalty0 261--272, 2020.
\newblock \doi{10.1038/s41592-019-0686-2}.

\bibitem[Wang et~al.(2024)Wang, Chen, Tang, Wu, and Zhu]{wang2024disentangled}
Xin Wang, Hong Chen, Si'ao Tang, Zihao Wu, and Wenwu Zhu.
\newblock Disentangled representation learning.
\newblock \emph{IEEE Transactions on Pattern Analysis and Machine
  Intelligence}, 46\penalty0 (12):\penalty0 9677--9696, 2024.

\bibitem[Wang(2016)]{wang2016normalcoordinates}
Zuoqin Wang.
\newblock Lecture 14: Normal coordinates.
\newblock Course notes for Riemannian Geometry, University of Science and
  Technology of China, 2016.
\newblock URL
  \url{http://staff.ustc.edu.cn/~wangzuoq/Courses/16S-RiemGeom/Notes/Lec14.pdf}.

\bibitem[Watanabe(1960)]{5392532}
Satosi Watanabe.
\newblock Information theoretical analysis of multivariate correlation.
\newblock \emph{IBM Journal of Research and Development}, 4\penalty0
  (1):\penalty0 66--82, 1960.
\newblock \doi{10.1147/rd.41.0066}.

\bibitem[Wei et~al.(2024)Wei, Gupta, and Morgado]{wei2024towards}
Yibing Wei, Abhinav Gupta, and Pedro Morgado.
\newblock Towards latent masked image modeling for self-supervised visual
  representation learning.
\newblock In \emph{European Conference on Computer Vision}, pages 1--17.
  Springer, 2024.

\bibitem[Williams et~al.(2025)Williams, Yu, Luu, Arvanitidis, and
  Klami]{williams2025geodesic}
Bernardo Williams, Hanlin Yu, Hoang Phuc~Hau Luu, Georgios Arvanitidis, and
  Arto Klami.
\newblock Geodesic slice sampler for multimodal distributions with strong
  curvature.
\newblock \emph{arXiv preprint arXiv:2502.21190}, 2025.

\bibitem[Wu et~al.(2021)Wu, Lischinski, and Shechtman]{wu2021stylespace}
Zongze Wu, Dani Lischinski, and Eli Shechtman.
\newblock Stylespace analysis: Disentangled controls for stylegan image
  generation.
\newblock In \emph{Proceedings of the IEEE/CVF Conference on Computer Vision
  and Pattern Recognition (CVPR)}, pages 12863--12872, 2021.

\bibitem[Yang et~al.(2021)Yang, Liu, Chen, Shen, Hao, and
  Wang]{yang_causalvae_2021}
Mengyue Yang, Furui Liu, Zhitang Chen, Xinwei Shen, Jianye Hao, and Jun Wang.
\newblock {CausalVAE}: Disentangled representation learning via neural
  structural causal models.
\newblock In \emph{Proceedings of the IEEE/CVF Conference on Computer Vision
  and Pattern Recognition (CVPR)}, pages 9588--9597, 2021.

\bibitem[Zhang and Fletcher(2013)]{ppga}
Miaomiao Zhang and P~Fletcher.
\newblock Probabilistic principal geodesic analysis.
\newblock \emph{NIPS}, 01 2013.

\bibitem[Zhang et~al.(2019)Zhang, Xing, and Zhang]{zhang2019mixture}
Youshan Zhang, Jiarui Xing, and Miaomiao Zhang.
\newblock Mixture probabilistic principal geodesic analysis.
\newblock In \emph{International Workshop on Multimodal Brain Image Analysis},
  pages 196--208. Springer, 2019.

\bibitem[Zheng et~al.(2022)Zheng, Ng, and Zhang]{zheng2022identifiability}
Yujia Zheng, Ignavier Ng, and Kun Zhang.
\newblock On the identifiability of nonlinear ica: Sparsity and beyond.
\newblock \emph{Advances in neural information processing systems},
  35:\penalty0 16411--16422, 2022.

\end{thebibliography}

\newpage
\appendix

\section{Background on Riemannian geometry}
\label{app:geometry-background}

We briefly review the tools from Riemannian geometry that we use in Riemannian ICA, following \cite{Lee00,lee2018introduction}.
Throughout this appendix we use Einstein summation convention, so repeated upper and lower coordinate indices are summed unless stated otherwise.

\paragraph{Smooth manifolds}
A \emph{smooth manifold} $\M$ \cite[Chapter 2]{Lee00} is a space that can be locally parameterized by a set of coordinates on an open subset of $\R^n$.
When referring to an abstract manifold, we write $\p \in \M$ for a point on the manifold.
We reserve $x$ for the coordinate system in which we observe data, so that $\x = x(\p)$ represents a data point.
Regions of a manifold can be parameterized by different coordinate systems, and by definition of a smooth manifold, there exists a diffeomorphism between any two coordinate systems called a transition map.
For example, if $x$ and $z$ are two coordinate systems on a manifold $\M$, then the transition map between $\z = z(\p)$ and $\x = x(\p)$ is given by $\x = \phi(\z)$, where $\phi = x \circ z^{-1}$.

A tangent vector $v \in T_\p\M$ \cite[Chapter 3]{Lee00} is a geometric object representing a direction of infinitesimal change on the manifold that lives in the tangent space $T_\p\M$ at the point $\p$.
In the $x$-coordinate system, we let $v_x \in \R^n$ denote the column vector of components of $v$, so that $v = v_x^i \frac{\partial}{\partial x^i}$.
Tangent vectors are coordinate-independent objects, so in a different coordinate system $z$, we have that $v = v_z^i \frac{\partial}{\partial z^i}$, where $v_x = \frac{\partial \phi(\z)}{\partial \z}\, v_z$, equivalently $v_z = \frac{\partial \phi(\z)}{\partial \z}^{-1}v_x$, and $\phi(\z)$ denotes the transition map from $z$-coordinates to $x$-coordinates.

\paragraph{Riemannian manifolds}
A \emph{Riemannian manifold} \cite{lee2018introduction} is a smooth manifold equipped with a \emph{Riemannian metric} \cite[Chapter 13]{Lee00}.
The metric assigns to each point $\p \in \M$ a positive-definite inner product $g_\p \colon T_\p\M \times T_\p\M \to \R$, and this assignment varies smoothly with $\p$.
We denote the component matrix of $g_\p$ in $x$-coordinates at $\x = x(\p)$ by $g_x(\x)$.
The inner product of two tangent vectors $u, v \in T_\p\M$ is then $g(u, v) = u_x^\top g_x(\x)\, v_x$.
We write $g(u, v)$ rather than $g_\p(u, v)$ because the basepoint $\p$ is already carried by $u$ and $v$.
These components can be computed in the $z$-coordinate system as $g_z(\z) = \frac{\partial \phi(\z)}{\partial \z}^\top g_x(\x)\, \frac{\partial \phi(\z)}{\partial \z}$.

A Riemannian metric induces a notion of orthogonality.
Two tangent vectors $u, v \in T_\p\M$ are $g$-orthogonal if $g(u,v) = 0$.
A set of tangent vectors $(v^{(1)}, \ldots, v^{(n)})$ at $\p$ is $g$-orthonormal if $g(v^{(i)}, v^{(j)}) = \delta_{ij}$, so distinct vectors are $g$-orthogonal and each vector has unit $g$-norm.
An orthonormal basis of $T_\p\M$ is a set of $g$-orthonormal tangent vectors that span the tangent space.
Suppose we extract the components of a $g$-orthonormal basis of tangent vectors and arrange them into a matrix $J_x$, where the $i$-th column contains the component vector $v_x^{(i)} \in \R^n$.
Then we say that $J_x$ is a $g$-orthogonal matrix and it satisfies $J_x^\top g_x(\x)\, J_x = I_n$.
This equation also shows that its left inverse is $J_x^{-1} = J_x^\top g_x(\x)$.

\paragraph{Tensors}
The tangent vectors $u, v$ and the metric $g$ are \emph{tensors}.
A tensor is a geometric object that exists independently of any coordinate system.
Each tensor has a type $(k, l)$, where $k$ is the number of upper indices and $l$ the number of lower indices of its components.
A tangent vector is a $(1,0)$-tensor, and so its component form is written as $v = v_x^i \frac{\partial}{\partial x^i}$.
The metric is a $(0,2)$-tensor, and so its component form is written as $g = g_{x,ij} dx^i \otimes dx^j$.
The type fixes how the components of a tensor transform under a change of coordinates to keep the underlying object fixed while its components change.
We have already seen this transformation rule for tangent vectors and the metric.
The components of a $(1,0)$-tensor transform as $v_z = \frac{\partial \phi(\z)}{\partial \z}^{-1} v_x$, and the components of a $(0,2)$-tensor transform as $g_z(\z) = \frac{\partial \phi(\z)}{\partial \z}^\top g_x(\x)\, \frac{\partial \phi(\z)}{\partial \z}$.
Together these rules keep the inner product invariant, so that $g(u,v) = u_x^\top g_x(\x)\, v_x = u_z^\top g_z(\z)\, v_z$.

\paragraph{Geodesics}
A \emph{geodesic}, under the Levi-Civita connection \cite[Chapter 5]{lee2018introduction}, is a curve $\x(t)$ with an initial velocity $v \propto \frac{\partial \x(t)}{\partial t}\big|_{t=0}$, $\| v \|_g = 1$, that satisfies the geodesic equation \cite[Theorem 4.10]{lee2018introduction}
\begin{equation}
  \label{app:eq:geodesic}
  \frac{\partial^2 \x(t)}{\partial t^2}^k + \Gamma_x(\x(t))\, ^k_{ij}\, \frac{\partial \x(t)}{\partial t}^i \frac{\partial \x(t)}{\partial t}^j = 0 \quad \forall k = 1, \ldots, n
\end{equation}
where $\Gamma_x(\x)^k_{ij}$ are the \emph{Christoffel symbols} of $g$ in the $x$-coordinate system, defined by
\begin{equation}
  \label{app:eq:christoffel}
  \Gamma_x(\x)^k_{ij} = \frac{1}{2}{g_x^{-1}(\x)}^{kl}\left(\frac{\partial g_x(\x)_{jl}}{\partial x^i} + \frac{\partial g_x(\x)_{il}}{\partial x^j} - \frac{\partial g_x(\x)_{ij}}{\partial x^l}\right)
\end{equation}
The Christoffel symbols express how neighboring points of the manifold are ``connected'' (via the Levi-Civita connection \cite[Chapter 5]{lee2018introduction}) so that it is possible to define an abstract notion of a directional derivative on a Riemannian manifold that does not require a choice of coordinate system \cite[Chapter 4]{lee2018introduction}.

\paragraph{Exponential and logarithmic maps}
The exponential map $\exp_{\x_0}(v_x)$ is the point on the geodesic that is $1$ unit of time away from $\x_0$ with initial velocity $v_x$.
It can be computed by simulating the geodesic equation with initial velocity $v_x$ and integrating for $1$ unit of time.
The corresponding geodesic between nearby endpoints can also be characterized variationally.
For a nearby endpoint $\x$, let $\x^*(t)$ denote the energy-minimizing curve from $\x_0$ to $\x$.
\begin{equation}
  \label{app:eq:geodesic-variational}
  \x^*(\cdot) = \arg \min_{\bm{\gamma} \colon [0,1] \to \M} \int_0^1 \frac{\partial \bm{\gamma}(t)^\top}{\partial t}\, g_x(\bm{\gamma}(t))\, \frac{\partial \bm{\gamma}(t)}{\partial t}\, dt
  \quad \text{subject to } \bm{\gamma}(0) = \x_0 \text{ and } \bm{\gamma}(1) = \x
\end{equation}
In a normal neighborhood of $\x_0$, this curve is unique, and the logarithmic map recovers its initial velocity.
\begin{equation}
  \label{app:eq:log-map-initial-velocity}
  \log_{\x_0}(\x) = \frac{\partial \x^*(t)}{\partial t}\bigg|_{t=0}
\end{equation}
Equivalently, the exponential map is the inverse operation that follows the unique geodesic with initial velocity $v_x$ for $1$ unit of time.

\paragraph{Riemann normal coordinates}
\label{app:sec:rnc}
Let $\p_0 \in \M$ with coordinates $\x_0$ and let $J_x \in \R^{n \times n}$ be a $g$-orthogonal matrix at $\x_0$, satisfying $J_x^\top g_x(\x_0)\, J_x = I_n$.
The columns of $J_x$ form an orthonormal basis of $T_{\p_0}\M$ under the metric $g$.
Riemann normal coordinates \cite[Chapter 5]{lee2018introduction} form a coordinate system in which variations of the coordinates trace geodesics whose initial velocity vectors are the columns of $J_x$.
In a normal neighborhood of $\p_0$, the normal coordinates, denoted by $\s$, are defined by $\s = h_{\scriptstyle \x_0, J_x}(\x) := J_x^{-1}\log_{\x_0}(\x)$, where $\log_{\x_0}$ is the logarithmic map.
At the origin, the value and first two derivatives of the normal coordinate chart are given by
\begin{equation}
  \label{app:eq:rnc-map}
  h_{\scriptstyle \x_0, J_x}(\x_0) = \zero, \quad
  \frac{\partial h_{\scriptstyle \x_0, J_x}(\x_0)}{\partial \x} = J_x^{-1}, \quad
  \frac{\partial^2 h_{\scriptstyle \x_0, J_x}^i(\x_0)}{\partial x^j \partial x^k} = \frac{\partial h_{\scriptstyle \x_0, J_x}^i(\x_0)}{\partial x^l}\, \Gamma_x(\x_0)^l_{jk}
\end{equation}
This means that given a coordinate description of a geometric object at a particular point, it is possible to explicitly construct the normal coordinate chart to perform computations.
An important consequence of this is that the inverse of the normal coordinate chart, $h_{\scriptstyle \x_0, J_x}^{-1}(\s) \mapsto \x$, can be treated as a decoder from normal coordinates to the data space.
At the origin, the Jacobian of this decoder is $J_x$:
\begin{equation}
  \label{app:eq:rnc-decoder-jacobian}
  \frac{\partial h_{\scriptstyle \x_0, J_x}^{-1}(\zero)}{\partial \s} = J_x, \quad \text{where } \zero = h_{\scriptstyle \x_0, J_x}(\x_0)
\end{equation}

\paragraph{Ricci Curvature Tensor}
\label{app:sec:ricci-curvature-tensor}
The Ricci curvature tensor is a $(0,2)$-tensor that is defined by
\begin{align}
  \label{app:eq:ricci-curvature-tensor}
  \Ric_x(\x)_{ij} &= g_x(\x)^{kl}\, R_x(\x)_{kijl} \\
  &= \frac{\partial \Gamma_x(\x)^k_{ij}}{\partial x^k} - \frac{\partial \Gamma_x(\x)^k_{ik}}{\partial x^j} + \Gamma_x(\x)^a_{ab}\,\Gamma_x(\x)^b_{ij} - \Gamma_x(\x)^a_{ib}\,\Gamma_x(\x)^b_{aj}
\end{align}
where $R_x(\x)_{kijl}$ is the $k,i,j,l$ component of the Riemann curvature tensor evaluated at $\x$.
We will not discuss the Riemann curvature tensor here.
See \cite[Chapter 7]{lee2018introduction} for more details.
The curvature tensor appears in the second derivative of the metric components in normal coordinates at the origin.
In particular, the metric in normal coordinates at the origin is Euclidean to first order so that its value and first two derivatives are given by \cite{wang2016normalcoordinates}
\begin{equation}
  \label{app:eq:rnc-metric}
  g_s(\zero)_{ij} = \delta_{ij},
  \quad
  \frac{\partial g_s(\zero)_{ij}}{\partial s^k} = 0,
  \quad
  \frac{\partial^2 g_s(\zero)_{ij}}{\partial s^k \partial s^l}
  = \frac{1}{3}\bigl(R_s(\zero)_{kilj} + R_s(\zero)_{likj}\bigr)
\end{equation}
This expression can be used to derive the value and first two derivatives of the log determinant of the metric in normal coordinates at the origin, which are given by
\begin{align}
  \label{app:normal-coordinate-metric}
  \frac{1}{2}\log \det g_s(\zero) = 0,
  \quad
  \frac{\partial \frac{1}{2} \log \det g_s(\zero)}{\partial \s} = \zero,
  \quad
  \frac{\partial^2 \frac{1}{2} \log \det g_s(\zero)}{\partial \s \partial \s^\top} = -\frac{1}{3}\Ric_s(\zero)
\end{align}
Through this equation and how it is used in \Cref{app:thm:disentanglement-tensor}, we interpret Ricci curvature as specifying how the change of variables formula for normal coordinates on a Riemannian manifold differs from its Euclidean counterpart.

\paragraph{Covariant Hessian}
\label{app:sec:covariant-hessian}
For a smooth scalar function $\psi \colon \M \to \R$, the \emph{covariant Hessian} $\covhess \psi(\p_0)$ at a basepoint $\p_0 \in \M$ is a $(0,2)$-tensor.
It is given by
\begin{equation}
  \label{app:eq:cov-hessian}
  \covhess \psi(\p_0)(u, v)
  = \Bigl(\frac{\partial^2 \psi(\x_0)}{\partial x^i \partial x^j}
    - \Gamma_x(\x_0)^k_{ij}\,\frac{\partial \psi(\x_0)}{\partial x^k}\Bigr)
    u_x^i\, v_x^j = \frac{\partial^2 \psi(\zero)}{\partial s^i \partial s^j}\, u_s^i\, v_s^j
\end{equation}
The covariant Hessian is the coordinate-invariant generalization of the Hessian of a function on a Euclidean space.
In normal coordinates at the origin the Christoffel symbols vanish, so it reduces to the ordinary Hessian, as the second equality in \Cref{app:eq:cov-hessian} shows.

\paragraph{Probability density on a manifold}
\label{app:sec:hausdorff-density}
Let $(\M, g)$ be a Riemannian manifold with Riemannian volume element $dV_g$.
A probability measure $\mu$ on $\M$ that is absolutely continuous with respect to the volume measure has a Hausdorff density $\rho \colon \M \to \R_{\geq 0}$ satisfying $d\mu = \rho\, dV_g$.
The probability that a random point $\p$ lies in a measurable set $D \subseteq \M$ is
\begin{equation}
  \label{app:eq:density}
  \mathbb{P}(\p \in D) = \int_D \rho\, dV_g
  = \mathbb{P}(\x \in D_x) = \int_{D_x} \rho_x(\x)\, \sqrt{\det g_x(\x)}\, d\x
\end{equation}
where $D_x$ is the image of $D$ under the coordinate map $x$.

\section{Proofs}
\label{app:proofs}

\PointwiseChangeOfVariablesFormula*
\label{app:prop:pointwise-change-of-variables-formula}
\begin{proof}
  The invariant object that we need to start with is the Hausdorff density.
  \begin{align}
    \rho_x(\x_0) = p_x(\x_0)/\sqrt{\det g_x(\x_0)}
  \end{align}
  $\rho$ is coordinate-free because the numerator and denominator both transform the same way.
  In normal coordinates $\s = h_{\scriptstyle \x_0, J_x}(\x)$, we have
  \begin{align}
    \rho_s(\zero) = p_s(\s)/\sqrt{\det g_s(\zero)}
  \end{align}
  where $\rho_s(\zero) = \rho_x(\x_0)$.
  This gives us the change of variables formula:
  \begin{align}
    \label{app:eq:general-change-of-variables}
    \log p_s(\zero) = \log p_x(\x_0) - \frac{1}{2}\log \det g_x(\x_0) + \frac{1}{2}\log \det g_s(\zero)
  \end{align}
  Since the metric in normal coordinates is the identity at the origin, the term $\frac{1}{2}\log \det g_s(\zero)=0$.
  The last step is to show how the determinant of $J_x$ relates to the determinant of the metric components.
  Since $J_x^\top g_x(\x_0) J_x = I_n$, we have that $\frac{1}{2}\log \det g_x(\x_0) = -\log |\det J_x|$.
  Plugging these into \Cref{app:eq:general-change-of-variables} yields the result.
\end{proof}

\DisentanglementTensor*
\label{app:thm:disentanglement-tensor}
\begin{proof}
  The result can be obtained simply by computing the Hessian of \Cref{eq:normal-change-of-variables-formula}.
  First, we note that derivatives in normal coordinates are given by the covariant derivative at the origin.
  \begin{align}
    \log p_s(\zero) &= \log p_x(\x_0) - \frac{1}{2}\log \det g_x(\x_0) + \frac{1}{2}\log \det g_s(\zero) \\
    &= \log \rho_x(\x_0) + \frac{1}{2}\log \det g_s(\zero)
  \end{align}
  For any scalar function $\psi$ on $\M$, the $s$-coordinate expression of $\psi$ is $\psi_s(\s) = \psi_x(\x(\s))$, and because Christoffel symbols vanish at $\s = \zero$, \Cref{app:eq:cov-hessian} reduces to
  \begin{equation}
    \frac{\partial^2 \psi_s(\zero)}{\partial \s \partial \s^\top} = (\covhess \psi)_s(\zero)
  \end{equation}

  From here, we apply the Hessian:
  \begin{align}
    \frac{\partial^2 \log p_s(\zero)}{\partial \s \partial \s^\top} &= \frac{\partial^2 \log \rho_x(x(\zero))}{\partial \s \partial \s^\top} + \frac{1}{2}\frac{\partial^2 \log \det g_s(\zero)}{\partial \s \partial \s^\top}
  \end{align}
  Since $\rho$ is a coordinate-free quantity, $\frac{\partial^2 \log \rho_x(x(\zero))}{\partial \s \partial \s^\top}$ contains components of $\nabla^2 \log \rho$ in the normal coordinate system at the origin, and using \Cref{app:normal-coordinate-metric} we get that the log determinant term is $-\frac{1}{3}\Ric$ in normal coordinates.
  This proves the result.
\end{proof}

\DisentanglementTensorCoordinates*
\label{app:prop:disentanglement-tensor-coordinates}
\begin{proof}
  Recall that the disentanglement tensor is defined as:
  \begin{equation}
    \mathcal{D}_x(\x_0) = \nabla^2 \log \rho_x(\x_0) - \tfrac{1}{3}\Ric_x(\x_0)
  \end{equation}
  From \Cref{app:eq:ricci-curvature-tensor}, the Ricci tensor expands as:
  \begin{align}
    \Ric_x(\x_0)_{ij} &= \frac{\partial \Gamma_x(\x_0)^k_{ij}}{\partial x^k} - \frac{\partial \Gamma_x(\x_0)^k_{ik}}{\partial x^j} + \Gamma_x(\x_0)^a_{ab}\,\Gamma_x(\x_0)^b_{ij} - \Gamma_x(\x_0)^a_{ib}\,\Gamma_x(\x_0)^b_{aj}
  \end{align}

  An alternative expression that is more useful for computation is:
  \begin{equation}
    \Ric_x(\x_0)_{ij} = \frac{\partial \Gamma_x(\x_0)^k_{ij}}{\partial x^k} - \Gamma_x(\x_0)^a_{ib}\,\Gamma_x(\x_0)^b_{aj} - \frac{1}{2}\covhess\log\det g_x(\x_0)_{ij}
  \end{equation}
  where $\covhess\log\det g_x(\x_0)_{ij}$ is the covariant Hessian \Cref{app:sec:covariant-hessian} of the log determinant of the metric.
  This follows from the identity that $\Gamma_x(\x)^a_{ai} = \frac{1}{2}\frac{\partial \log\det g_x(\x)}{\partial x^i}$.
  With this, we can verify this by direct computation:
  \begin{align}
    \frac{1}{2}\covhess\log\det g_x(\x_0)_{ij} &= \frac{\partial \Gamma_x(\x_0)^a_{ai}}{\partial x^j} - \Gamma_x(\x_0)^a_{ab}\,\Gamma_x(\x_0)^b_{ij}
  \end{align}
  The final step is expanding the Hausdorff density in terms of the log-density and the metric and grouping terms together.
  \begin{align}
    \mathcal{D}_x(\x_0)_{ij} &= \nabla^2 \log \rho_x(\x_0)_{ij} - \tfrac{1}{3}\Ric_x(\x_0)_{ij} \\
    &= \left(\nabla^2 \log \frac{p_x(\x_0)}{\sqrt{\det g_x(\x_0)}}\right)_{ij} - \frac{1}{3}\left(\frac{\partial \Gamma_x(\x_0)^k_{ij}}{\partial x^k} - \Gamma_x(\x_0)^a_{ib}\,\Gamma_x(\x_0)^b_{aj} - \frac{1}{2}\covhess\log\det g_x(\x_0)_{ij}\right) \\
    &= \nabla^2 \log p_x(\x_0)_{ij} - \frac{1}{3}\covhess\log\det g_x(\x_0)_{ij} - \frac{1}{3}\frac{\partial \Gamma_x(\x_0)^k_{ij}}{\partial x^k} + \frac{1}{3}\Gamma_x(\x_0)^a_{ib}\,\Gamma_x(\x_0)^b_{aj}
  \end{align}
  Expanding the covariant Hessian term gives the first part of the proposition.

  The second part relating the components of the disentanglement tensor in the data space and normal coordinates comes from the tensor transformation rule for $(0,2)$-tensors, and the fact that the Jacobian matrix of $h_{\scriptstyle \x_0, J_x}^{-1}(\s) \mapsto \x$ is $J_x$ at $\s = \zero$ (\Cref{app:eq:rnc-decoder-jacobian}).
  This gives us the result:
  \begin{align}
    \mathcal{D}_s(\zero) &= {J_x}^\top \mathcal{D}_x(\x_0) {J_x}
  \end{align}
\end{proof}

\PointwiseDisentanglementRICA*
\label{app:thm:rica}
\begin{proof}
  Let $J^*_x$ be the solution to the generalized eigenvalue problem, ${J^*_x}^\top \mathcal{D}_x(\x_0)\, J^*_x = \Lambda$ where $\Lambda$ is diagonal and ${J^*_x}^\top g_x(\x_0)\, J^*_x = I_n$.
  Let $h_{\scriptstyle \x_0, J^*_x}$ be the normal coordinate chart with orientation $J^*_x$, let $p_s = (h_{\scriptstyle \x_0, J^*_x})_\# p_x$, and let $\mathcal{D}_s(\zero)$ be the components of $\mathcal{D}$ in this chart.
  Since $\mathcal{D}$ is a $(0,2)$-tensor, its components in the $s$-coordinates are
  \begin{equation}
    \mathcal{D}_s(\zero) = {J^*_x}^\top \mathcal{D}_x(\x_0)\, J^*_x = \Lambda
  \end{equation}
  By \Cref{thm:disentanglement-tensor},
  \begin{equation}
    \frac{\partial^2 \log p_s(\zero)}{\partial \s \partial \s^\top}
    = \mathcal{D}_s(\zero)
    = \Lambda
  \end{equation}
  Therefore the Hessian of the latent log-likelihood is diagonal at the origin, so $h_{\scriptstyle \x_0, J^*_x}$ pointwise disentangles $p_x$ at $\x_0$ by \Cref{def:local-independence}.

  When the entries of $\Lambda$ are distinct, the generalized eigendecomposition is unique up to sign flips and permutations \cite{golub2013matrix}.
  Therefore any other solution has the form $\tilde J_x = J^*_x Q$, where $Q$ is a signed permutation matrix.
  For any $\x$ in the normal coordinate neighborhood, the corresponding latents are related by
  \begin{equation}
    \tilde \s
    = \tilde J_x^{-1}\log_{\x_0}(\x)
    = Q^{-1}{J^*_x}^{-1}\log_{\x_0}(\x)
    = Q^{-1}\s.
  \end{equation}
  Thus the recovered latents are identifiable up to sign flips and permutations.
\end{proof}

\section{Experimental details}
\label{app:experimental-details}

\subsection{Smoothed Monge metric used in Figure~\ref{fig:rica-overview-rica}}
\label{app:figure1-monge-metric}

The RICA panel of \Cref{fig:rica-overview} uses a smoothed variant of the Monge metric of \cite{hartmann2022lagrangian} rather than the literal form $g(\x) = I + \alpha^2\, \nabla \log q(\x)\, \nabla \log q(\x)^\top$, where $q(\x)$ is the density learned by the RealNVP \cite{dinh2017density} in \Cref{fig:rica-overview-rica}.
This was purely for aesthetic purposes as computing the exponential and logarithmic maps in the region displayed was numerically unstable, and we emphasize that the alignment of the geodesic ball with the structure of the data depends on the choice of metric.
We use a smooth metric that is a weighted sum of RealNVP score outer products:
\begin{equation}
  \label{eq:figure1-monge-metric}
  g_x(\x) = I + \alpha^2 \sum_{m=1}^{1000} w_m(\x)\, s_m\, s_m^\top,
  \qquad
  w_m(\x) = \frac{\exp\bigl(-\|\x - \mathbf{a}_m\|^2 / (2 h^2)\bigr)}{\sum_{m'=1}^{1000} \exp\bigl(-\|\x - \mathbf{a}_{m'}\|^2 / (2 h^2)\bigr)}
\end{equation}
The anchor points $\{\mathbf{a}_m\}_{m=1}^{1000} \subset \R^2$ were selected around the two-moons cloud in regions where the RealNVP density is well behaved.
At each anchor point, we computed the score $s_m = \nabla \log q(\mathbf{a}_m)$ using the pretrained RealNVP density.
The scalar $\alpha$ controls the strength of the score term and $h$ sets the bandwidth of the weights $w_m$.
The plotted panel uses $\alpha^2 = 0.05$ and $h = 0.15$.

\subsection{Source recovery experimental details}
\label{app:source-recovery-experimental-details}

This appendix contains the details of the source recovery experiment in \Cref{sec:source-recovery}.
All runs use the generative model in \Cref{eq:tangent-source-gen-model}.
Unless stated otherwise, we use $N = 10{,}000$ samples per draw, where one draw is a random $g$-orthonormal matrix $J_x$ at the fixed base point together with the samples generated from it. We average over $B = 5$ such matrices and $5$ independent random seeds, with base scale $b = 0.3$ and source decay $r_s = 0.85$.
All methods are evaluated on the same samples for each draw.

\paragraph{Manifolds and charts.}
\label{app:source-recovery-experimental-details:manifolds-and-charts}
RICA is evaluated in the intrinsic coordinates $\x$ for each manifold.
FastICA and conformal NL-ICA are evaluated either on these intrinsic coordinates or on alternate coordinates $\z = \phi(\x)$ specified by each table row.
\begin{enumerate}
  \item \emph{Sphere $S^n$}.
  The intrinsic coordinates $\x \in \R^n$ are stereographic coordinates from the north pole, with metric
  \[
    g_x(\x) = 4(1 + \|\x\|^2)^{-2} I_n
  \]
  This is the round metric of the unit sphere in stereographic coordinates, which \cite[Lemma~3.4]{lee2018introduction} obtains as the pullback of the Euclidean metric on $\R^{n+1}$ through the embedding $\phi$ below.
  We use the base point represented by $\x_0 = \zero$ in this chart.
  The alternate coordinates are the ambient embedding $\z = \phi(\x) \in \R^{n+1}$,
  \[
    \phi(\x) =
    \left(\frac{2\x}{1 + \|\x\|^2},\, \frac{\|\x\|^2 - 1}{1 + \|\x\|^2}\right)
  \]
  \item \emph{Hyperbolic space $H^n$}.
  The intrinsic coordinates $\x \in \R^n$ are Poincar\'e ball coordinates, with metric
  \[
    g_x(\x) = 4(1 - \|\x\|^2)^{-2} I_n, \qquad \|\x\| < 1
  \]
  This is the Poincar\'e ball model of unit-curvature hyperbolic space \cite[Theorem~3.5]{lee2018introduction}.
  We use the base point represented by $\x_0 = \zero$ in this chart.
  The alternate coordinates are the Lorentz embedding $\z = \phi(\x) \in \R^{n+1}$,
  \[
    \phi(\x) =
    \left(\frac{2\x}{1 - \|\x\|^2},\, \frac{1 + \|\x\|^2}{1 - \|\x\|^2}\right)
  \]
  \item \emph{Flat torus $T^n$}.
  The intrinsic coordinates $\x \in (-\pi, \pi]^n$ are angle coordinates, with metric
  \[
    g_x(\x) = I_n
  \]
  We use the base point represented by $\x_0 = \zero$ in this chart.
  The alternate coordinates are the sine-cosine embedding $\z = \phi(\x) \in \R^{2n}$,
  \[
    \phi(\x) = (\sin x_1, \cos x_1, \ldots, \sin x_n, \cos x_n)
  \]
  \item \emph{SPD($p$) under the affine-invariant metric}.
  Points are positive definite matrices $P \in \mathrm{SPD}(p)$, and the metric at $P$ is
  \[
    g_P(U,V) = \tr(P^{-1} U P^{-1} V)
  \]
  for tangent matrices $U,V$.
  The intrinsic coordinates are $\x = \operatorname{vech}(\log P) \in \R^n$, where $\log P$ is the matrix logarithm, $\operatorname{vech}$ stacks lower-triangular entries into a vector, and $n = p(p+1)/2$.
  We write $\operatorname{unvech}$ for the inverse map from vector coordinates back to symmetric matrices.
  We use the scaled $\operatorname{vech}$ basis, with diagonal entries weighted by $1$ and off-diagonal entries weighted by $1/\sqrt{2}$.
  Let $P(\x) = \exp(\operatorname{unvech}(\x))$ and $E_i(\x) = \partial P(\x) / \partial x_i$.
  The coordinate representation of the metric is
  \[
    g_x(\x)_{ij} = g_{P(\x)}(E_i(\x), E_j(\x))
  \]
  We use the matrix base point $P_0 = I_p$.
  Its intrinsic coordinate is $\x_0 = \operatorname{vech}(\log I_p) = \zero$, and $g_x(\zero) = I_n$.
  The alternate coordinates are $\z = \phi(\x) \in \R^n$,
  \[
    \phi(\x) = \operatorname{vech}(P(\x))
  \]
\end{enumerate}
The intrinsic dimension is $n$ for sphere, hyperbolic space, and torus, and $n = p(p+1)/2$ for SPD($p$).

\paragraph{Methods.}
\emph{Closed-form RICA} uses the exact local disentanglement tensor in data coordinates, $\mathcal{D}_x(\x_0)$.
For the tangent-source model, the same coordinate matrix can be computed directly from the normal-coordinate tensor $\mathcal{D}_s(\zero)$ as
\begin{equation}
  \mathcal{D}_x(\x_0) = J_x^{-\!\top}\,\mathcal{D}_s(\zero)\,J_x^{-1}
  = J_x^{-\!\top}\,
  \frac{\partial^2 \log p_s(\zero)}{\partial \s\,\partial \s^\top}\,
  J_x^{-1}
\end{equation}
The paper results use this identity as a computational shortcut, and we checked that it agrees with the full derivative-based construction wherever both are tractable.
It recovers $J_x^*$ from the eigendecomposition of $\mathcal{D}_x(\x_0)$ and returns $\hat\s = (J_x^*)^{-1} \log_{\x_0}(\x)$.
\emph{FastICA} runs FastICA \cite{hyvarinen2000ica} on the centered chart input.

\emph{Conformal NL-ICA} restricts the nonlinear mixing function to a conformal map, whose Jacobian has the form $\lambda(\x)O(\x)$ for a scalar function $\lambda$ and an orthogonal matrix field $O$.
The identifiability theorem of \citet{buchholz2022function} assumes dimension $n>2$, independent source coordinates, at most one Gaussian source coordinate, and a mild regularity condition on each source marginal.
Under these assumptions, the conformal ICA model is identifiable up to shifts, signed coordinate permutations, and a shared scale.
Our baseline parameterizes this conformal map as a composition of $L = 8$ M\"obius transformations,
\begin{equation}
  \label{eq:conformal-nlica-baseline}
  \x = f_L \circ \cdots \circ f_1(\s),
  \qquad
  f_\ell(\s) =
  \alpha_\ell O_\ell
  \frac{\s - a_\ell}{\|\s - a_\ell\|^2}
  + b_\ell.
\end{equation}
where $\s$ denotes latent variables, $\x \in \mathbb{R}^{n}$ denotes the chart input, and $\alpha_\ell$, $O_\ell$, $a_\ell$, and $b_\ell$ are the layer parameters used by the implemented map.
We note that each individual M\"obius transformation in \Cref{eq:conformal-nlica-baseline} does not represent the full set of M\"obius transformations because the M\"obius class includes transformations that do not perform the sphere inversion (i.e. divide by $\frac{1}{\|\s - a_\ell\|^2}$).
However, a composition of two layers of \Cref{eq:conformal-nlica-baseline} can undo the sphere inversion when the second layer's inversion center matches the previous layer's output translation.
In practice, we also needed to restrict the parameters for numerical stability so that the denominator does not underflow.

We optimized unconstrained scalars $\tilde\eta_\ell$, unconstrained matrices $A_\ell$, and unconstrained raw vectors $\tilde a_\ell,\tilde b_\ell \in \mathbb{R}^{n}$.
The parameters used in \Cref{eq:conformal-nlica-baseline} are computed via $\alpha_\ell = \exp(4\tanh(\tilde\eta_\ell))$, $O_\ell = \exp(A_\ell - A_\ell^\top)$, $a_\ell = 2\sqrt{n}\tilde a_\ell / \|\tilde a_\ell\|$, and $b_\ell = 2\sqrt{n}\tilde b_\ell / \|\tilde b_\ell\|$.

We train this M\"obius flow by maximum likelihood under a logistic prior whose scales are matched to the source scales of the generative model.
Training runs for $30{,}000$ AdamW steps at batch size $128$, and we return latents from the trained encoder.

NL-ICA performs worse than FastICA on nearly every cell, and the gap widens at $n \approx 64$ because of the spread of the source scales.
With $r_s = 0.85$ the scales $b_i$ are spread over two orders of magnitude at $n = 32$ and over four at $n = 64$.
FastICA does not have this issue because it whitens the observations first, which puts every source on a common scale before separation.
NL-ICA, on the other hand, is trained by maximum likelihood directly on the raw anisotropic data and is only able to recover the large-scale sources.

\paragraph{Sweeps.}
The main table uses $n = 32$ for sphere, hyperbolic space, and torus.
For SPD in the main table, $p = 8$, so the samples are $8 \times 8$ positive-definite matrices with intrinsic dimension $n = 36$.
The eigengap sweep in \Cref{fig:source-recovery-eigvalgap} fixes target dimension $n = 32$ and varies $r_s \in \{0.999, 0.99, 0.97, 0.95, 0.9, 0.85, 0.75, 0.6\}$ at $b = 0.3$, with $p = 8$ for SPD.
The base-scale sweep in \Cref{fig:source-recovery-b-sweep} uses the same target dimension and varies $b \in \{0.05, 0.1, 0.2, 0.35, 0.5, 0.75, 1.0, 1.5, 2.0\}$ at fixed $r_s = 0.85$.
Increasing $b$ pushes samples farther from the base point, which exposes the injectivity-radius failure mode described in the main text.

\paragraph{Higher-dimension table.}
\Cref{tab:source-recovery-n64} reports the same source-recovery setup as \Cref{tab:source-recovery} in the main text but at target intrinsic dimension $n \approx 64$, with sphere, hyperbolic space, and torus at $n = 64$ and SPD at $11 \times 11$ matrices with intrinsic dimension $n = 11(11+1)/2 = 66$.
RICA remains at MCC $\approx 1.0$ on every cell, while the conformal NL-ICA baseline drops from MCC $\approx 0.85$ on the smooth intrinsic charts at $n = 32$ to MCC $\approx 0.50$ at $n = 64$.

\begin{table}[t]
\centering
\resizebox{\textwidth}{!}{%
\begin{tabular}{l l r r @{\hspace{1.5em}} r r @{\hspace{1.5em}} r r}
\toprule
 & & \multicolumn{2}{c}{RICA} & \multicolumn{2}{c}{FastICA} & \multicolumn{2}{c}{NL-ICA} \\
\cmidrule(lr){3-4} \cmidrule(lr){5-6} \cmidrule(lr){7-8}
Manifold & Chart & MCC $\uparrow$ & TC $\downarrow$ & MCC $\uparrow$ & TC $\downarrow$ & MCC $\uparrow$ & TC $\downarrow$ \\
\midrule
Sphere $S^d$ & ambient & \textbf{1.00 $\pm$ 0.00} & \textbf{0.02 $\pm$ 0.06} & 0.91 $\pm$ 0.01 & 1.02 $\pm$ 0.16 & 0.53 $\pm$ 0.01 & 12.04 $\pm$ 0.64 \\
 & stereographic & \textbf{1.00 $\pm$ 0.00} & \textbf{0.02 $\pm$ 0.06} & 0.95 $\pm$ 0.02 & 0.85 $\pm$ 0.07 & 0.50 $\pm$ 0.03 & 13.64 $\pm$ 0.78 \\
\midrule
Hyperbolic $H^d$ & Lorentz & \textbf{1.00 $\pm$ 0.00} & \textbf{0.02 $\pm$ 0.06} & 0.95 $\pm$ 0.00 & 1.45 $\pm$ 0.08 & 0.51 $\pm$ 0.01 & 13.41 $\pm$ 0.47 \\
 & Poincar\'e & \textbf{1.00 $\pm$ 0.00} & \textbf{0.02 $\pm$ 0.06} & 0.98 $\pm$ 0.00 & -0.25 $\pm$ 0.06 & 0.53 $\pm$ 0.01 & 12.27 $\pm$ 0.67 \\
\midrule
Flat torus $T^d$ & angle & \textbf{1.00 $\pm$ 0.00} & \textbf{0.02 $\pm$ 0.06} & 0.98 $\pm$ 0.00 & -0.09 $\pm$ 0.05 & 0.52 $\pm$ 0.01 & 12.85 $\pm$ 0.65 \\
 & sincos & \textbf{1.00 $\pm$ 0.00} & \textbf{0.02 $\pm$ 0.06} & 0.37 $\pm$ 0.01 & 7.47 $\pm$ 0.11 & 0.27 $\pm$ 0.01 & 12.12 $\pm$ 0.25 \\
\midrule
SPD($p$) & log-Euclidean & \textbf{1.00 $\pm$ 0.00} & \textbf{0.01 $\pm$ 0.05} & 0.98 $\pm$ 0.00 & -0.09 $\pm$ 0.05 & 0.50 $\pm$ 0.01 & 14.06 $\pm$ 0.57 \\
 & vech & \textbf{1.00 $\pm$ 0.00} & \textbf{0.01 $\pm$ 0.05} & 0.26 $\pm$ 0.01 & 8.96 $\pm$ 0.09 & 0.32 $\pm$ 0.00 & 13.11 $\pm$ 0.31 \\
\bottomrule
\end{tabular}
}
\vspace{4pt}
\caption{Source recovery at target intrinsic dimension $n \approx 64$.
  Same setup as \Cref{tab:source-recovery} but with the intrinsic dimension doubled.
  Sphere, hyperbolic space, and torus use $n = 64$, while SPD uses $11 \times 11$ positive-definite matrices, which have intrinsic dimension $n = 11(11+1)/2 = 66$.
  RICA recovers the ground truth sources on every cell, while the conformal NL-ICA baseline degrades sharply compared to the $n \approx 32$ setting.}
\label{tab:source-recovery-n64}
\end{table}

\paragraph{Software assets}
\label{app:software-assets}
The experiments use synthetic data generated by scikit-learn \cite{scikit-learn}, numerical arrays and linear algebra from NumPy \cite{harris2020array} and SciPy \cite{virtanen2020scipy}, automatic differentiation and vectorization from JAX \cite{jax2018github}, neural network modules from Equinox \cite{kidger2021equinox}, optimization with Optax \cite{deepmind2020jax}, and plotting with Matplotlib \cite{Hunter:2007}.
We used the \texttt{local\_coordinates} library \cite{cunningham2026localcoordinates} as a reference for computing all of the geometric quantities used in our paper.
All of these are open source software packages.
The relevant project licenses are BSD 3-Clause for scikit-learn, NumPy, and SciPy, Apache License 2.0 for JAX, Equinox, Optax, and \texttt{local\_coordinates}, and the PSF based Matplotlib license.

\paragraph{LLM usage}
\label{app:llm-usage}
Large language models were used to help edit the writing and implement the experiments.

\newpage

\end{document}